\documentclass{svjour3}
\usepackage[ruled, vlined, nofillcomment, linesnumbered]{algorithm2e}
\usepackage{cite}
\usepackage{url}
\usepackage{amsmath}
\usepackage{amssymb}
\usepackage{graphicx}
\usepackage{subcaption}
\usepackage[english]{babel}
\usepackage{booktabs}
\usepackage{nicefrac}
\usepackage{verbatim}
\usepackage{float}
\usepackage{xspace}
\usepackage{tikz}
\usepackage{pgfplots}
\usepackage[misc]{ifsym}
\usetikzlibrary{backgrounds,trees}
\usepgfplotslibrary{fillbetween}

\usepackage[numbers,compress]{natbib}

\renewcommand{\refname}{References}
\makeatletter
\renewcommand\bibsection{%
  \section*{{\refname}\@mkboth{\refname}{\refname}}%
}%
\makeatother

\SetKw{Output}{output}

\newcommand{\set}[1]{\left\{#1\right\}}
\newcommand{\pr}[1]{\left(#1\right)}
\newcommand{\fpr}[1]{\mathopen{}\left(#1\right)}

\newcommand{\abs}[1]{{\left|#1\right|}}

\newcommand{\enset}[2]{\left\{#1 ,\ldots , #2\right\}}
\newcommand{\enpr}[2]{\pr{#1 ,\ldots , #2}}

\newcommand{\real}{\mathbb{R}}

\newcommand{\funcdef}[3]{{#1}:{#2} \to {#3}}
\newcommand{\define}{\leftarrow}

\DeclareRobustCommand{\dispfunc}[2]{%
	\ensuremath{%
		\ifthenelse{\equal{#2}{}}%
			{\mathit{#1}}%
			{\mathit{#1}\fpr{#2}}}}

\newcommand{\score}[1]{\dispfunc{s}{#1}}
\newcommand{\lab}[1]{\dispfunc{\ell}{#1}}

\newcommand{\weight}[1]{\dispfunc{d}{#1}}
\newcommand{\cweight}[1]{\dispfunc{cd}{#1}}
\newcommand{\auc}[1]{\dispfunc{auc}{#1}}
\newcommand{\lcount}[1]{\dispfunc{lcount}{#1}}

\newcommand{\hm}[1]{\dispfunc{h}{#1}}
\newcommand{\chm}[1]{\dispfunc{ch}{#1}}

\newcommand{\nxt}[1]{\dispfunc{next}{#1}}

\newcommand{\lc}[1]{\dispfunc{left}{#1}}
\newcommand{\rc}[1]{\dispfunc{right}{#1}}
\newcommand{\troot}[1]{\dispfunc{root}{#1}}

\newcommand{\diff}[1]{\dispfunc{d}{#1}}
\newcommand{\cdiff}[1]{\dispfunc{cd}{#1}}
\newcommand{\shift}[1]{\dispfunc{s}{#1}}

\newcommand{\bigO}[1]{\dispfunc{\mathcal{O}}{#1}}

\newcommand{\dtname}[1]{\textsl{#1}}

\newcommand{\algauc}{\textsc{DynAuc}\xspace}
\newcommand{\alghexact}{\textsc{Hexact}\xspace}
\newcommand{\alghapprox}{\textsc{Happrox}\xspace}

\newcommand{\algsubset}{\textsc{Subset}\xspace}
\newcommand{\algsubsetalt}{\textsc{SubsetAlt}\xspace}

\SetKw{And}{and}
\SetKw{Or}{or}
\SetKw{Not}{not}

\pgfdeclarelayer{background}
\pgfdeclarelayer{foreground}
\pgfsetlayers{background,main,foreground}

\definecolor{yafaxiscolor}{rgb}{0.3, 0.3, 0.3}

\definecolor{yafcolor1}{rgb}{0.4, 0.165, 0.553}
\definecolor{yafcolor2}{rgb}{0.949, 0.482, 0.216}
\definecolor{yafcolor3}{rgb}{0.47, 0.549, 0.306}
\definecolor{yafcolor4}{rgb}{0.925, 0.165, 0.224}
\definecolor{yafcolor5}{rgb}{0.141, 0.345, 0.643}
\definecolor{yafcolor6}{rgb}{0.965, 0.933, 0.267}
\definecolor{yafcolor7}{rgb}{0.627, 0.118, 0.165}
\definecolor{yafcolor8}{rgb}{0.878, 0.475, 0.686}

\newlength{\yafaxispad}
\newlength{\yaftlpad}
\newlength{\yaflabelpad}
\newlength{\yafaxiswidth}
\newlength{\yafticklen}
\makeatletter
\def\pgfplots@drawtickgridlines@INSTALLCLIP@onorientedsurf#1{}
\makeatother

\newcommand{\yafdrawaxis}[4]{
	\pgfplotstransformcoordinatex{#1}\let\xmincoord=\pgfmathresult 
	\pgfplotstransformcoordinatex{#2}\let\xmaxcoord=\pgfmathresult 
	\pgfplotstransformcoordinatey{#3}\let\ymincoord=\pgfmathresult 
	\pgfplotstransformcoordinatey{#4}\let\ymaxcoord=\pgfmathresult 
	\pgfsetlinewidth{\yafaxiswidth} 
	\pgfsetcolor{yafaxiscolor}
	\pgfpathmoveto{\pgfpointadd{\pgfpointadd{\pgfplotspointrelaxisxy{0}{0}}{\pgfqpointxy{\xmincoord}{0}}}{\pgfqpoint{-0.5\yafaxiswidth}{\yafaxispad}}}
	\pgfpathlineto{\pgfpointadd{\pgfpointadd{\pgfplotspointrelaxisxy{0}{0}}{\pgfqpointxy{\xmaxcoord}{0}}}{\pgfqpoint{0.5\yafaxiswidth}{\yafaxispad}}}
	\pgfpathmoveto{\pgfpointadd{\pgfpointadd{\pgfplotspointrelaxisxy{0}{0}}{\pgfqpointxy{0}{\ymincoord}}}{\pgfqpoint{\yafaxispad}{-0.5\yafaxiswidth}}}
	\pgfpathlineto{\pgfpointadd{\pgfpointadd{\pgfplotspointrelaxisxy{0}{0}}{\pgfqpointxy{0}{\ymaxcoord}}}{\pgfqpoint{\yafaxispad}{0.5\yafaxiswidth}}}
	\pgfusepath{stroke}
}

\pgfplotscreateplotcyclelist{yaf}{% 
{yafcolor5,mark options={scale=0.75},mark=o}, 
{yafcolor2,mark options={scale=0.75},mark=square},
{yafcolor3,mark options={scale=0.75},mark=triangle},
{yafcolor4,mark options={scale=0.75},mark=o},
{yafcolor1,mark options={scale=0.75},mark=o},
{yafcolor6,mark options={scale=0.75},mark=o},
{yafcolor7,mark options={scale=0.75},mark=o},
{yafcolor8,mark options={scale=0.75},mark=o}} 

\pgfkeys{/pgf/number format/.cd,1000 sep={\,}}

\pgfplotsset{axis y line=left, axis x line=bottom,
	tick align=outside,
	tickwidth=\yafticklen,
	clip = false,
    x axis line style= {-, line width = 0pt, color=black!0},
    y axis line style= {-, line width = 0pt, color=black!0},
    x tick style= {line width = \yafaxiswidth, color=yafaxiscolor, yshift = \yafaxispad},
    y tick style= {line width = \yafaxiswidth, color=yafaxiscolor, xshift = \yafaxispad},
    x tick label style = {font=\scriptsize, yshift = \yaftlpad, inner xsep = 0pt},
    y tick label style = {font=\scriptsize, xshift = \yaftlpad},
    every axis y label/.style = {at = {(ticklabel cs:0.5)}, rotate=90, anchor=center, font=\scriptsize, yshift = -\yaflabelpad, inner sep = 0pt},
    every axis x label/.style = {at = {(ticklabel cs:0.5)}, anchor=center, font=\scriptsize, yshift = \yaflabelpad},
    x tick label style = {font=\scriptsize, yshift = 1pt},
    grid = major,
    major grid style  = {dash pattern = on 1pt off 3 pt},
	every axis plot post/.append style= {line width=\yafaxiswidth} ,
	legend cell align = left,
	legend style = {inner sep = 1pt, cells = {font=\scriptsize}},
	legend image code/.code={%
		\draw[mark repeat=2,mark phase=2,#1] 
		plot coordinates { (0cm,0cm) (0.15cm,0cm) (0.3cm,0cm) };% 
	} 
}

\begin{document}

\title{Maintaining AUC and $H$-measure over time}

\author{Nikolaj Tatti}

\institute{HIIT, University of Helsinki, Helsinki, Finland\\
\email{nikolaj.tatti@helsinki.fi}}

\date{Received: date / Accepted: date}

\maketitle              % typeset the header of the contribution

\begin{abstract}

Measuring the performance of a classifier is a vital task in machine learning.
The running time of an algorithm that computes the measure plays a very small
role in an offline setting, for example, when the classifier is being developed by a researcher.
However, the running time becomes more crucial if our goal is to monitor the performance of a classifier
over time.

In this paper
we study three algorithms for maintaining two measures.
The first algorithm maintains area under the ROC curve (AUC) under addition and deletion of data points in $\bigO{\log n}$ time.
This is done by maintaining the data points sorted in
a self-balanced search tree. In addition, we augment the search tree that allows us to query the ROC
coordinates of a data point in $\bigO{\log n}$ time. In doing so we are able to maintain AUC in $\bigO{\log n}$ time.
Our next two algorithms involve in maintaining $H$-measure, an alternative measure based on the ROC curve.
Computing the measure is a two-step process: first we need to compute a convex hull of the ROC curve, followed
by a sum over the convex hull. We demonstrate that we can maintain the convex hull using a minor modification 
of the classic convex hull maintenance algorithm. We then show that 
under certain conditions, we can compute the $H$-measure exactly in $\bigO{\log^2 n}$ time, and if the conditions
are not met, then we can estimate the $H$-measure in $\bigO{(\log n + \epsilon^{-1})\log n}$ time.
We show empirically that our methods are significantly faster than the baselines.

\keywords{AUC, $H$-measure, online algorithm}
\end{abstract}

\section{Introduction}

Measuring the performance of a classifier is a vital task in machine learning.
The running time of an algorithm that computes the measure plays a very small
role in an offline setting, for example, when the classifier is being developed
by a researcher.  However, the running time becomes more crucial if our goal is
to monitor the performance of a classifier over time where the new data points
may arrive at a significant speed.

For example, consider a task of monitoring abnormal behaviour in IT systems
based on event logs.  Here, the main problem is the gargantuan volume of event
logs making the manual monitoring impossible.  One approach is to have a
classifier to monitor for abnormal events and alert analysts for closer
inspection.  Here, monitoring should be done continuously to notice
abnormalities rapidly.  Moreover, the performance of the classifier should also
be monitored continuously as the underlying distribution, and potentially the
performance of the classifier, may change due to the changes in the IT system.

In order to detect recent changes in the performance,
we are often interested in the performance over the last $n$ data points.
More generally, we are interested in maintaining the measure under addition or deletion of data points.

We study algorithms for maintaining two measures. The first measure is the area under the ROC curve (AUC), a classic
technique of measuring the performance of a classifier based on its ROC curve. We also study
$H$-measure, an alternative measure proposed by~\citet{hand:09:alternative}.
Roughly speaking, the measure is based on the minimum weighted loss, averaged over the cost ratio.
A practical advantage of the $H$-measure over AUC is that it allows a natural way of weighting
classification errors.

Both measures can be computed in $\bigO{n \log n}$ time from scratch, or in $\bigO{n}$ time if the data points are already sorted. 
In this paper we present 3 algorithms that allow us to maintain the measures in polylogarithmic time.

The first algorithm maintains AUC under addition or deletion of data points.
The approach is straightforward: we maintain the data points sorted in a
self-balanced search tree.  In order to update AUC we need to know the ROC
coordinates of the data point that we are changing.  Luckily, this can be done
by modifying the search tree so that it maintains the cumulative counts of the labels in
each subtree.  Consequently, we can obtain the coordinates in $\bigO{\log n}$
time, which leads to a total of $\bigO{\log n}$ maintenance time.

Our next two algorithms involve maintaining the $H$-measure. Computing the $H$-measure
involves finding the convex hull of the ROC curve, and enumerating over the hull.
First we show that we can use a classic dynamic convex hull algorithm with some minor
modifications to maintain the convex hull of the ROC curve. The modifications are required
as we do not have the ROC coordinates of individual data points, but we can use the same
trick as when computing AUC to obtain the needed coordinates.

Then we show that if we estimate the class priors from the test data, we can
decompose the $H$-measure into a sum over the points in the convex hull such that
the $i$th term depends only on the difference between the $i$th and the $(i - 1)$st data points.
This decomposition allows us to maintain the $H$-measure in $\bigO{\log^2 n}$ time.

If the class priors are \emph{not} estimated from the test data, then we propose an estimation
algorithm. Here the idea is to group points that are close in the convex hull together. Or in other words,
if there are points in the convex hull that are close to each other, then we only use one data point from such group.
The grouping is done in a way that we maintain $\epsilon$-approximation in $\bigO{(\log n + \epsilon^{-1}) \log n}$ time.

\textbf{Structure:}
The rest of the paper is organized as follows. We present preliminary definitions in Section~\ref{sec:prel}.
In Section~\ref{sec:auc} we demonstrate how to maintain AUC, and in Sections~\ref{sec:hexact}--\ref{sec:happrox}
we demonstrate how to maintain the $H$-measure. We present the experimental evaluation in Section~\ref{sec:exp},
and conclude the paper with a discussion in Section~\ref{sec:conclusions}.

\section{Preliminaries}\label{sec:prel}

Assume that we are given a multiset of $n$ data points $Z$. Each
data point $z = (s, \ell)$ consists of a score $s \in R$ and a true label
$\ell \in \set{1, 2}$.
The score is typically obtained by applying a classifier with
high values implying that $z$ should be classified as class 2. 
To simplify the notation greatly, given $z = (s, \ell)$ we define $\weight{z} = (1, 0)$ if $\ell = 1$,
and $\weight{z} = (0, 1)$ if  $\ell = 2$. We can now write
\[
	(n_1, n_2) = \sum_{z \in Z} \weight{z},
\]
that is, $n_j$ is the number of points having the label equal to $j$.
Here we used a convention
that the sum of two tuples, say $(a, b)$ and $(c, d)$, is $(a + c, b + d)$.
Note that $n = n_1 + n_2$.

Let $S = \enpr{s_1}{s_n}$ be the list of all scores, ordered from the smallest to the largest.
Let us write 
\begin{equation}
\label{eq:rocraw}
	r_i = \sum_{z \in Z, \score{z} \leq s_i} \weight{z},
\end{equation}
that is, $r_i$ are the label counts of points having a score less than or equal to $s_i$.

We obtain the \emph{ROC curve} by normalizing $r_i$ in Eq.~\ref{eq:rocraw}, that is,
the ROC curve is a list of $n + 1$ points $X = (x_0, x_1, \ldots, x_n)$, where
\[
	x_i = (r_{i1} / n_1, r_{i2} / n_2)
\]
and $x_0 = (0, 0)$.  Note that not all points in $X$ are necessarily unique.
The points in $X$ are confined in the unit rectangle of $(0, 1) \times (0, 1)$.
See Figure~\ref{fig:toyauc} for illustration.\footnote{For notational
convenience, we treat the first coordinate as the vertical and the second
coordinate as the horizontal.} 

\begin{figure}
\begin{center}
\begin{tikzpicture}
\begin{axis}[xlabel={}, ylabel= {},
    width = 6.2cm,
    scaled x ticks = false,
    cycle list name=yaf,
    yticklabel style={/pgf/number format/fixed},
    xticklabel style={/pgf/number format/fixed},
	xlabel = {$x_{i2}$},
	ylabel = {$x_{i1}$},
    no markers,
    legend pos = {north west}
    ]
\addplot+[name path=R] coordinates {(0, 0) (0.05, 0.2) (0.2, 0.6) (0.4, 0.8) (0.6, 0.9) (1, 1)} node[pos=0.5, black, sloped, auto] {\small ROC};

\path[name path=axis] (axis cs:0,0) -- (axis cs:1,0);

\addplot[yafcolor5, opacity=0.1] fill between[of=R and axis,soft clip={domain=0:1}];

\node at (axis cs:0.6,0.4) {AUC};

\pgfplotsextra{\yafdrawaxis{0}{1}{0}{1}}
\end{axis}
\end{tikzpicture}
\end{center}

\caption{Example of a ROC curve and AUC. If we consider label 1 as a true label and
label 2 as a false label, then the vertical axis is the true positive rate (TPR)
while the horizonal axis is the false positive rate (FPR).}
\label{fig:toyauc}

\end{figure}
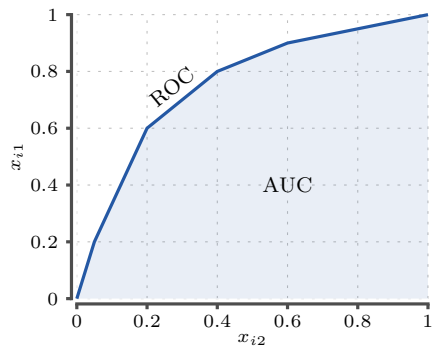

The area under the curve, $\auc{Z}$ is the area below the ROC curve.
If there is a threshold $\sigma$ such that all data points with a score smaller
than $\sigma$ belong to class 1 and all data points with a score larger than
$\sigma$ belong to class 2, then $\auc{Z} = 1$. If the scores are independent
of the true labels, then the expected value of $\auc{Z}$ is $1/2$.

Instead of defining $\auc{Z}$ using the ROC curve, we can also define it directly
with Mann-Whitney $U$ statistic~\citep{mann1947test}. Assume that we are given a multiset of points $Z$.
Let $S_1 = \set{s \mid (s, \ell) \in Z, \ell = 1}$ be a multiset of scores with
the corresponding labels being equal to 1, and define $S_2$ similarly. The Mann-Whitney $U$ statistic is equal to
\begin{equation}
\label{eq:mann}
	U = \sum_{s \in S_1} \sum_{t \in S_2} f(s, t), \quad\text{where}\quad
	f(s, t) =
	\begin{cases}
	1 & \text{ if } s < t, \\
	0.5 & \text{ if } s = t, \\
	0 & \text{ if } s > t\quad. \\
	\end{cases}
\end{equation}
We obtain $\auc{Z}$ by normalizing $U$, that is, $\auc{Z} = \frac{1}{\abs{S_1}\abs{S_2}}U$.

AUC can be computed naively using $U$ statistic in $\bigO{n^2}$ time. However, we can easily
speed up the computation
to $\bigO{n \log n}$ time using Algorithm~\ref{alg:aucoffline}. 
To see the correctness, note that in Eq.~\ref{eq:mann} each $t \in S_2$ contributes to $U$
with 
\[
	\sum_{s \in S_1} f(s, t) = \abs{\set{s \in S_1 \mid s < t}} + \frac{1}{2}\abs{\set{s \in S_1 \mid s = t}}\quad.
\]
Algorithm~\ref{alg:aucoffline} achieves its running time by maintaining the first term (in a variable $h$) as it loops over sorted scores.
Note that if $Z$ is already sorted, then the running time reduces to linear.

\begin{algorithm}
\caption{Algorithm for computing $\auc{Z}$}
\label{alg:aucoffline}
$S \define $ unique scores of $Z$, sorted\;
$(n_1, n_2) \define $ label counts\;
$U \define 0$; $h \define 0$\;
\ForEach {$s \in S$} {
	$(w_1, w_2) \define \sum_{\score{z} = s} \weight{z}$\;
	$U \define U + w_2(h + w_1/2)$\;
	$h \define h + w_1$\;
}
\Return $U / (n_1n_2)$\;
\end{algorithm}

Our first goal is to show that we can maintain AUC in $\bigO{\log n}$ time
under addition or removal of data points.

Our second contribution is a procedure for maintaining $H$-measure.
	
$H$-measure is an alternative method proposed by~\citet{hand:09:alternative}. The main
idea is as follows: consider minimizing weighted loss,
\[
\begin{split}
	Q(c, \sigma) & = c p(\score{z} > \sigma, \lab{z} = 1) + (1 - c) p(\score{z} \leq \sigma, \lab{z} = 2) \\
	             & = c \pi_1 p(\score{z} > \sigma \mid \lab{z} = 1) + (1 - c) \pi_2 p(\score{z} \leq \sigma \mid \lab{z} = 2), \\
\end{split}
\]
where $c$ is a cost ratio, $\sigma$ is a threshold, $z$ is a random data point, and $\pi_k = p(\lab{z} = k)$ are class priors.
Let us write $\sigma(c)$ to be the 
threshold minimizing $Q(c, \sigma)$ for a given $c$. Increasing $c$
will decrease $\sigma(c)$, or in other words by varying $c$ we will
vary the threshold. As pointed out by~\citet{flach2011coherent} the curve $Q(c, \sigma(c))$ is
a variant of a \emph{cost curve} (see~\citep{drummond2006cost}),
\[
	c p(\score{z} > \sigma \mid \lab{z} = 1) + (1 - c) p(\score{z} \leq \sigma \mid \lab{z} = 2)\quad. 
\]
Here the difference
is that $Q(c, \sigma(c))$ uses class priors $\pi_k$ whereas the cost curve
omits them.

Since not all values of $c$ may be sensible, we assume that we are given a
weight function $u(c)$.
We are interested in measuring the weighted minimum loss as we vary $c$,
\begin{equation}
\label{eq:lintegral}
	L = \int Q(c, \sigma(c)) u(c) dc\quad.
\end{equation}
Here small values of $L$ indicate strong signal between the labels and the score. 

The $H$-measure is a normalized version of $L$,
\[
	H = 1 - L / L_{\textit{max}}\quad.
\]
Here, $L_{\textit{max}}$ is the largest possible value of $L$ over all possible ROC curves.
The negation is done so that the values of $H$ are consistent with the AUC scores:
values close to 1 represent good performance.

We will see that the convenient choice for $u$ will be a beta distribution, as
suggested by~\citet{hand:09:alternative}, since it allows us to express the
integrals in a closed form.

Computing the empirical $H$-measure in practice starts with an ROC curve $X$.
The following computations assume that the ROC curve is convex. If not, then
the first step is to compute the convex hull of $X$, which we will denote by $Y = \enpr{y_0}{y_m}$.
Taking a convex hull will inflate the performance of the underlying classifier, however it is
possible to modify the underlying classifier (see~\citep{hand:09:alternative} for more details) so that its ROC curve
is convex.

We then define
\begin{equation}
\label{eq:slope}
	c_i = \frac{\pi_2(y_{i2} - y_{(i - 1)2})}{\pi_2(y_{i2} - y_{(i - 1)2)}) + \pi_1(y_{i1} - y_{(i - 1)1})},
\end{equation}
where, recall that, $\pi_k = p(\lab{z} = k)$ are the class probabilities and $\enpr{y_0}{y_m}$ is the convex hull. The probabilities $\pi_k$ can be either estimated
from $Z$ or by some other means. If former, then we show that we can maintain the $H$-measure
exactly, if latter, then we need to estimate the measure in order to achieve a sublinear maintenance time.

We also set $c_0 = 0$ and $c_m = 1$. Note that $c_i$ is a monotonically decreasing function of the slope of the convex hull.
This guarantees that $c_i \leq c_{i + 1}$. We can show that (see~\citep{hand:09:alternative}) if $c_i < c  < c_{i + 1}$, then 
the minimum loss is equal to
\[
	Q(c, \sigma(c)) = c\pi_1(1 - y_{i1}) + (1 - c)\pi_2y_{i2} \quad.
\]

We can now write Eq.~\ref{eq:lintegral} as
\begin{equation}
\label{eq:lgeneric}
	L = \sum_{i = 0}^m \pi_1(1 - y_{i1}) \int_{c_i}^{c_{i + 1}} c u(c) dc + \pi_2y_{i2} \int_{c_i}^{c_{i + 1}} (1 - c)u(c)dc,
\end{equation}
and if we use beta distribution with parameters $(\alpha, \beta)$ as $u(c)$, we have
\begin{equation}
\label{eq:lbeta}
\begin{split}
	L = \frac{1}{B(1, \alpha, \beta)}\sum_{i = 0}^m & \pi_1(1 - y_{i1}) \pr{B(c_{i + 1}; \alpha + 1, \beta) - B(c_i; \alpha + 1, \beta)} \\
	 & + \pi_2y_{i2} \pr{B(c_{i + 1}; \alpha, \beta + 1) - B(c_i; \alpha, \beta + 1)},
\end{split}
\end{equation}
where $B(\cdot, \alpha, \beta)$ is an incomplete beta function.

Finally, we can show that the normalization constant is equal to
\[
	L_{\textit{max}} = \frac{\pi_1B(\pi_1; \alpha + 1, \beta) + \pi_2 B(1; \alpha, \beta + 1) - \pi_2 B(1; \alpha, \beta + 1)}{B(1, \alpha, \beta)}\quad.
\]

Given an ROC curve $X$, computing the convex hull $Y$, and subsequent steps, can be done in $\bigO{n}$ time.
We will show in Section~\ref{sec:hexact} that we can maintain the $H$-measure in $\bigO{\log^2 n}$ time if $\pi_k$ are estimated from $Z$. Otherwise
we will show in Section~\ref{sec:happrox} that we can approximate the $H$-measure in $\bigO{(\epsilon^{-1} + \log n)\log n}$ time.

As pointed earlier, $Q(c, \sigma(c))$ can be viewed as a variant of a cost
curve. If we were to replace $Q$ with the cost curve and use uniform
distribution for $u$, then, as pointed by~\citet{flach2011coherent}, $L$ is equivalent to the area
under the cost curve. Interestingly enough, we cannot use the algorithm given
in Section~\ref{sec:hexact} to compute the area of under the cost curve as the
precense of the priors is needed to decompose the measure. However, we can use the algorithm in Section~\ref{sec:happrox} to estimate
the area under the cost curve.

Interestingly enough, $Q(c, \sigma)$ can be linked to AUC. If, instead of using
the optimal threshold $\sigma(c)$, we average $Q$ over carefully selected
distribution for $\sigma$ and also use uniform distribution for $c$, then the resulting integral
is a linear transformation of AUC~\citep{flach2011coherent}.

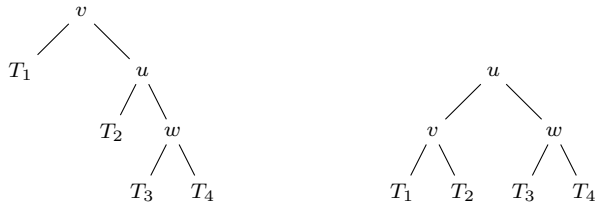
\begin{figure}
\begin{center}
\begin{tikzpicture}[level distance=0.8cm,
  level 1/.style={sibling distance=1.6cm},
  level 2/.style={sibling distance=0.8cm}]
  \node {$v$}
    child {node {$T_1$}}
    child {node {$u$}
      child {node {$T_2$}}
      child {node {$w$}
        child {node {$T_3$}}
        child {node {$T_4$}}
      }
	};
\end{tikzpicture}\hspace{2cm}
\begin{tikzpicture}[level distance=0.8cm,
  level 1/.style={sibling distance=1.6cm},
  level 2/.style={sibling distance=0.8cm}]
  \node {$u$}
    child {node {$v$}
      child {node {$T_1$}}
      child {node {$T_2$}}
    }
    child {node {$w$}
    child {node {$T_3$}}
      child {node {$T_4$}}
    };
\end{tikzpicture}
\end{center}

\caption{An example of left rotation in a search tree. Left figure: before
rotation, right figure: after rotation. Note that only $u$ and $v$ have
different children after the rotation.}
\label{fig:rotation}

\end{figure}

\textbf{Self-balancing search trees}
In this paper we make a significant use of self-balancing search trees such as
AVL-trees of red-black trees.  Such trees are binary trees where each node, say
$u$, has a key, say $k$. The left subtree of $u$ contains nodes with keys
smaller than $k$ and the right subtree of $u$ contains nodes with keys larger
than $k$. Maintaining this invariant allows for efficient queries as long as
the height of the tree is kept in check. Self-balancing trees such as AVL-trees
or red-black trees keep the height of the tree in $\bigO{\log n}$.
The balancing is done with $\bigO{\log n}$ number of left rotations or right rotations whenever the tree is modified (see Figure~\ref{fig:rotation}).
Searching for nodes with specific keys, inserting new nodes, and deleting existing
nodes can be done in $\bigO{\log n}$ time. Moreover, splitting the search
tree into two search tree or combining two trees into one can also be done in $\bigO{\log n}$ time.

We assume that we can compare and manipulate integers of
size $\bigO{n}$ and real numbers in constant time.  We do this because it is
reasonable to assume that the current bit-length of integers in modern computer
acrhitecture is sufficient for any practical applications, and we do need to
resort to any custom big integer implementations. If needed, however, the running
times need to be multiplied by an additional $\bigO{\log n}$ factor.  

\section{Related work}
Several works have studied maintaining AUC in a sliding window.
\citet{brzezinski:17:pauc} maintained the order of $n$ data points using a
red-black tree but computed AUC from scratch, resulting in a running time of
$\bigO{n + \log n}$, per update. \citet{tatti:18:auc} proposed algorithm yielding
$\epsilon$-approximation of AUC in $\bigO{(1 + \epsilon^{-1})\log n)}$ time, per update.
Here the approach bins the ROC space into a small number of bins. The bins are selected so that 
the AUC estimate is accurate enough.
\citet{bouckaert:06:auc} proposed estimating AUC by binning and only maintaining counters for
individual bins.
On the other hand, in this work we do not need to resort to binning,
instead we can maintain the exact AUC by maintaining a search search tree structure
in $\bigO{\log n}$ time, per update.

We should point out that AUC and the $H$-measure are defined over the whole ROC
curve, and are useful when we do not want to commit to a specific
classification threshold. On the other hand, if we do have the threshold, then
we can easily maintain a confusion matrix, and consequently maintain many
classic metrics, for example, accuracy, recall,
$F1$-measure~\cite{gama2013evaluating,gama2010knowledge}, and
Kappa-statistic~\citep{bifet2010sentiment,vzliobaite2015evaluation}.

In a related work, \citet{ataman2006learning,ferri2002learning,brefeld2005auc,herschtal2004optimising}
proposed methods where AUC is optimized as a part of training a classifier. Note that 
this setting differs from ours: changing the classifier parameters most likely will change
the scores of \emph{all} data points, and may change the data point order significantly. On the other hand,
we rely on the fact we can maintain the order using a search tree. Interestingly, \citet{calders:07:auc}
estimated AUC using a continuous function which then allowed optimizing the classifier parameters
with gradient descent.

Our approaches are useful if
we are working in a sliding window setting, that is, we want to compute
the relevant statistic using only the last $n$ data points. In other words, we 
abruptly forget the $(n + 1)$th data point. An alternative option would be to gradually
downplay the importance of older data points. A convenient option is to use exponential decay, see for example a survey by~\citet{gama2014survey}.
While maintaining the confusion matrix is trivial when using exponential decay
but---to our knowledge---there are no methods for maintaining AUC or $H$-measure under exponential decay. 

\section{Maintaining AUC}\label{sec:auc}

In this section we present a simple approach to maintain AUC in $\bigO{\log n}$ time.
We accomplish this by showing that the \emph{change} in AUC can be computed in $\bigO{\log n}$ time whenever
a new point is added or an existing point is deleted.
We rely on the following two propositions that express how AUC changes when adding or deleting a data point.
We then show that the quantities occurring in the propositions, namely, the weights $(u_1, u_2)$ and
$(v_1, v_2)$ can be obtained in $\bigO{\log n}$ time.

\begin{proposition}[Addition]
\label{prop:aucadd}
Let $Z$ be a set of data points with $(n_1, n_2)$ label counts.
Let $Y$ be a set of points having the same score $\sigma$.
Write $(w_1, w_2) = \sum_{y \in Y} \weight{y}$.
Define also 
\[
	(u_1, u_2) = \sum_{z \in Z \atop \score{z} < \sigma} \weight{z} \quad\text{and}\quad
	(v_1, v_2) = \sum_{z \in Z \atop \score{z} = \sigma} \weight{z} \quad .
\]
Write $U = n_1n_2 \times \auc{Z}$  and $U' = (n_1 + w_1)(n_2 + w_2) \times \auc{Z \cup Y}$.
Then
\[
	U' = U + w_2\pr{u_1 + \frac{v_1}{2}} + w_1\pr{n_2 - u_2 - \frac{v_2}{2}}  + \frac{w_1w_2}{2}\quad.
\]
\end{proposition}

\begin{proof}
We will use Mann-Whitney U statistic, given in Eq.~\ref{eq:mann} to prove the claim.
Let us write $Z' = Z \cup Y$ and define
\[
	S_i = \set{s \mid (s, \ell) \in Z, \ell = i} \quad\text{and}\quad S_i' = \set{s \mid (s, \ell) \in Z', \ell = i},
	\quad\text{for}\quad i = 1, 2\quad.
\]
Eq.~\ref{eq:mann} states that
\[
\begin{split}
	U' & = \sum_{s \in S_1'} \sum_{t \in S_2'} f(s, t) \\
	   & = w_1 \sum_{t \in S_2'} f(\sigma, t) + w_2 \sum_{s \in S_1} f(s, \sigma) + \sum_{s \in S_1} \sum_{t \in S_2} f(s, t)  \\
	   & = w_1 \sum_{t \in S_2'} f(\sigma, t) + w_2 \sum_{s \in S_1} f(s, \sigma) + U  \\
	   & = w_1 \pr{n_2 - u_2 - v_2 + \frac{v_2 + w_2}{2}} + w_2 \pr{u_1 + \frac{v_1}{2}} + U\quad.  \\
\end{split}
\]
We obtain the claim by rearranging the terms.\qed
\end{proof}

\begin{proposition}[Deletion]
\label{prop:aucremove}
Let $Z$ be a set of data points with $(n_1, n_2)$ label counts.
Let $Y \subseteq Z$ be a set of points having the same score $\sigma$.
Write $(w_1, w_2) = \sum_{y \in Y} \weight{y}$.
Define also 
\[
	(u_1, u_2) = \sum_{z \in Z \atop \score{z} < \sigma} \weight{z} \quad\text{and}\quad
	(v_1, v_2) = \sum_{z \in Z \atop \score{z} = \sigma} \weight{z} \quad .
\]
Write $U = n_1n_2 \times \auc{Z}$  and $U' = (n_1 - w_1)(n_2 - w_2) \times \auc{Z \setminus Y}$.
Then
\[
	U' = U - w_2\pr{u_1 + \frac{v_1}{2}} - w_1\pr{n_2 - u_2 - \frac{v_2}{2}}  + \frac{w_1w_2}{2}\quad.
\]
\end{proposition}

Note that the sign of the last term is the same for both addition and deletion.

\begin{proof}
We will use Mann-Whitney U statistic, given in Eq.~\ref{eq:mann} to prove the claim.
Let us write $Z' = Z \setminus Y$ and define
\[
	S_i = \set{s \mid (s, \ell) \in Z, \ell = i} \quad\text{and}\quad S_i' = \set{s \mid (s, \ell) \in Z', \ell = i},
	\quad\text{for}\quad i = 1, 2\quad.
\]
Eq.~\ref{eq:mann} states that
\[
\begin{split}
	U & = \sum_{s \in S_1} \sum_{t \in S_2} f(s, t) \\
	   & = w_1 \sum_{t \in S_2} f(\sigma, t) + w_2 \sum_{s \in S_1'} f(s, \sigma) + \sum_{s \in S_1'} \sum_{t \in S_2'} f(s, t)  \\
	   & = w_1 \sum_{t \in S_2} f(\sigma, t) + w_2 \sum_{s \in S_1'} f(s, \sigma) + U' \\
	   & = w_1 \pr{n_2 - u_2 - v_2 + \frac{v_2}{2}} + w_2 \pr{u_1 + \frac{v_1 - w_1}{2}} + U'\quad.  \\
\end{split}
\]
We obtain the claim by rearranging the terms.\qed
\end{proof}
Note that normally we would be adding or deleting a single data point, that is, $Y = \set{y}$. However,
the propositions also allow us to modify multiple points with the same score.

These two propositions allow us to maintain AUC as long as we can compute $(u_1, u_2)$ and $(v_1, v_2)$.
To compute these quantities we will use a balanced search tree $T$ such as red-black tree or AVL tree.
Let $S$ be the unique scores of $Z$. Each score $s \in S$ is given a node $n \in T$.

Moreover, for each node $x$ with a score of $s$, we will store the total label counts having the same score,
$\weight{x} = \sum_{\score{z} = s} \weight{z}$. The counts $\weight{x}$ will give us immediately $(v_1, v_2)$.

In addition, we will store $\cweight{x}$, cumulative label counts of all descendants of $x$, including $x$ itself.
We need to maintain these counts whenever we add or remove nodes from $T$, change the counts of nodes, or when $T$
needs to be rebalanced. Luckily, since 
\[
	\cweight{x} = \cweight{\lc{x}} + \cweight{\rc{x}} + \weight{x}
\]
we can compute $\cweight{x}$ in constant time as long as we have the cumulative counts of children of $x$.
Whenever node $x$ is changed, only its ancestors are changed, so the cumulative weights can be updated
in $\bigO{\log n}$ time. The balancing in red-black tree or AVL tree is done by using left or right rotation.
Only two nodes are changed per rotation (see Figure~\ref{fig:rotation}), and we can recompute the cumulative counts for these nodes in constant time.
There are at most $\bigO{\log n}$ rotations, so the running time is not increased.

Given a tree $T$ and a score threshold $\sigma$, let us define $\lcount{\sigma, T} = \sum_{\score{x} < \sigma} \weight{x}$, to be the total
count of nodes with scores smaller than $\sigma$. Computing $\lcount{s, T}$ gives us $(u_1, u_2)$ used by Propositions~\ref{prop:aucadd}--\ref{prop:aucremove}.

In order to compute $\lcount{\sigma, T}$ we will use the procedure given in Algorithm~\ref{alg:leftweight}. 
Here, we use a binary search over the tree, and summing the cumulative counts
of the left branch.  To see the correctness of the algorithm, observe that
during the while-loop Algorithm~\ref{alg:leftweight} maintains the invariant that
$u + \cweight{\lc{x}}$ is equal to $\lcount{\score{x}, T}$.
We should point out that similar queries were considered by~\citet{tatti:18:auc}. However, they were not combined
with Propositions~\ref{prop:aucadd}--\ref{prop:aucremove}.

\begin{algorithm}
\caption{Computes $\lcount{\sigma, T}$ using a binary search tree}
\label{alg:leftweight}

$x \define $ root of $T$\;
$u \define (0, 0)$\;

\While {$x$ \Or $\score{x} \neq \sigma$} {
	\uIf {$\score{x} > \sigma$} {
		$x \define \lc{n}$\;
	}
	\Else {
		$u \define u + \cweight{\lc{x}} + \weight{x}$\;
		$x \define \rc{x}$\;
	}
}
\Return $u + \cweight{\lc{x}}$\;
\end{algorithm}

Since $T$ is balanced, the running time of Algorithm~\ref{alg:leftweight} is $\bigO{\log n}$.

In summary, we can maintain $T$ in $\bigO{\log n}$ time, and we can obtain
$(u_1, u_2)$ and $(v_1, v_2)$ using $T$ in $\bigO{\log n}$ time.  These
quantities allow us to maintain AUC in $\bigO{\log n}$ time.

\section{Maintaining H-measure}
\label{sec:hexact}

If we were to compute the $H$-measure from scratch, we first need to compute the convex
hull, and then compute the $H$-measure from the convex hull. In order to maintain the
$H$-measure, we will first address
maintaining the convex hull, and then explain how we maintain the actual measure.

\subsection{Divide-and-conquer approach for maintaining a convex hull}

Maintaining a convex hull under point additions or deletions is a well-studied
topic in computational geometry. A classic approach by~\citet{overmars1981maintenance} maintains
the hull in $\bigO{\log^2 n}$ time. Luckily, the same approach with some modifications
will work for us.

Before we continue, we should stress two important differences between our
setting and a traditional setting of maintaining a convex hull.

First, in a normal setting, the additions and removals are done to new \emph{points in a
plane}.  In other words, the remaining points do not change over time.  In our
case, the data point consists of a classifier score and a label, and modifications shift
the ROC coordinates of every point. As a concrete example, in a traditional setting,
adding a point cannot reveal already existing points whereas adding a new data
point can shift the ROC curve enough so that some existing points become included
in the convex hull.

Secondly, we do not have the coordinates for all the points. However, it turns out
that we can compute the \emph{needed} coordinates with no additional costs.

We should point out that the approach by~\citet{overmars1981maintenance} is not the fastest for
maintaining the hull: for example an algorithm by~\citet{brodal2002dynamic} can maintain the hull in
$\bigO{\log n}$ time. However, due to the aforementioned differences adapting this
algorithm to our setting is non-trivial, and possibly infeasible.

We will explain next the main idea behind the algorithm by~\citet{overmars1981maintenance},
and then modify it to our needs.

The overall idea behind the algorithm is as follows.
A generic convex hull can be viewed as a union of the lower convex hull and the upper convex hull.
We only need to compute the upper convex hull, and for simplicity, we will refer to the upper convex hull
as the convex hull.

\begin{figure}

\begin{center}
\subcaptionbox{Combining partial hulls\label{fig:bridge}}{
\begin{tikzpicture}

\draw[yafcolor5, thick] (0, 0) node[fill, circle, inner sep = 0.5mm, label={[circle,black]0:$h_1$}] {}
                 -- ++(0.2, 1) node[fill, circle, inner sep = 0.5mm, label={[circle,black]0:$h_2$}] {}
				 -- ++(0.5, 0.6) node[fill, circle, inner sep = 0.5mm, label={[circle,black, inner sep = 0pt]-45:$h_3$}] (b1) {}
				 -- ++(0.7, 0.3) node[fill, circle, inner sep = 0.5mm, label={[circle,black, inner sep = 0pt]-90:$h_4$}] {}
				 -- ++(1, 0) node[fill, circle, inner sep = 0.5mm, label={[circle, black, inner sep = 0pt]-90:$h_5$}] {};

\draw[yafcolor5, thick] (3, 2) node[fill, circle, inner sep = 0.5mm, label={[circle,black]0:$g_1$}] {}
                 -- ++(0.2, 1) node[fill, circle, inner sep = 0.5mm, label={[circle,black]0:$g_2$}] {}
				 -- ++(0.5, 0.6) node[fill, circle, inner sep = 0.5mm, label={[circle,black, inner sep = 0pt]-45:$g_3$}] (b2) {}
				 -- ++(0.7, 0.3) node[fill, circle, inner sep = 0.5mm, label={[circle,black, inner sep = 0pt]-90:$g_4$}] {}
				 -- ++(1, 0) node[fill, circle, inner sep = 0.5mm, label={[circle, black, inner sep = 0pt]-90:$g_5$}] {};

\draw[yafcolor2, thick] (b1) -- (b2) node[pos=0.5, sloped, black, auto, inner sep = 0pt] {bridge};

\end{tikzpicture}}
\subcaptionbox{Search tree for hulls\label{fig:hulltree}}[4cm]{
\begin{tikzpicture}[scale=0.3]

\begin{scope}[]
\draw[yafcolor5, thick] (0, 0) node[fill, circle, inner sep = 0.3mm] (n1) {}
                 ++(0.2, 1) node[fill, circle, inner sep = 0.3mm] (n2) {}
				 ++(0.5, 0.6) node[fill, circle, inner sep = 0.3mm] (n3) {}
				 ++(0.7, 0.3) node[fill, circle, inner sep = 0.3mm] (n4) {}
				 ++(0.8, 0.2) node[fill, circle, inner sep = 0.3mm] (n5) {}
				 ++(1, 0) node[fill, circle, inner sep = 0.3mm] (n6) {};
\draw[thick] (n1) edge[yafcolor1] (n2);
\draw[thick] (n2) edge[yafcolor2] (n3);
\draw[thick] (n3) edge[yafcolor3] (n4);
\draw[thick] (n4) edge[yafcolor4] (n5);
\draw[thick] (n5) edge[yafcolor5] (n6);
\end{scope}

\begin{scope}[shift = {(-1, -4.7)}]
\draw[yafcolor5, thick] (0, 0) node[fill, circle, inner sep = 0.3mm] (n1) {}
                 ++(0.2, 1) node[fill, circle, inner sep = 0.3mm] (n2) {}
				 ++(0.5, 0.6) node[fill, circle, inner sep = 0.3mm] (n3) {}
				 ++(0.7, 0.3) node[fill, circle, inner sep = 0.3mm] (n4) {};
\draw[thick] (n1) edge[yafcolor1] (n2);
\draw[thick] (n2) edge[yafcolor2] (n3);
\draw[thick] (n3) edge[yafcolor3] (n4);

\begin{scope}[shift = {(-1.5, -4)}]
\draw[yafcolor5, thick] (0, 0) node[fill, circle, inner sep = 0.3mm] (n1) {}
                 ++(0.2, 1) node[fill, circle, inner sep = 0.3mm] (n2) {};
\draw[thick] (n1) edge[yafcolor1] (n2);
\end{scope}

\begin{scope}[shift = {(2, -4.5)}]
\draw[yafcolor5, thick] (0, 0) 
                 ++(0.2, 1) 
				 ++(0.5, 0.6) node[fill, circle, inner sep = 0.3mm] (n3) {}
				 ++(0.7, 0.3) node[fill, circle, inner sep = 0.3mm] (n4) {};
\draw[thick] (n3) edge[yafcolor3] (n4);
\end{scope}

\draw[thick] (1, 0) edge[in=90, out=-90] (-1, -2.5);
\draw[thick] (1, 0) edge[in=90, out=-90] (3, -2.5);

\end{scope}

\begin{scope}[shift = {(1, -5)}]
\draw[yafcolor5, thick] (0, 0) 
                 ++(0.2, 1) 
				 ++(0.5, 0.6) 
				 ++(0.7, 0.3) 
				 ++(0.8, 0.2) node[fill, circle, inner sep = 0.3mm] (n5) {}
				 ++(1, 0) node[fill, circle, inner sep = 0.3mm] (n6) {};
\draw[thick] (n5) edge[yafcolor5] (n6);
\end{scope}

\draw[thick] (2, 0) edge[in=90, out=-90] (0, -2.5);
\draw[thick] (2, 0) edge[in=90, out=-90] (4, -2.5);

\end{tikzpicture}}
\end{center}

\caption{Left figure: an example of combining two partial convex hulls into one by finding a bridge
segment. Right figure: a stylized data structure for maintaing convex hull. Each node corresponds to a partial convex hull (that are stored in separate search trees),
a parent hull is obtained from the child hulls by finding the bridge segment. Leaf nodes containing individual data points are not shown.}

\end{figure}
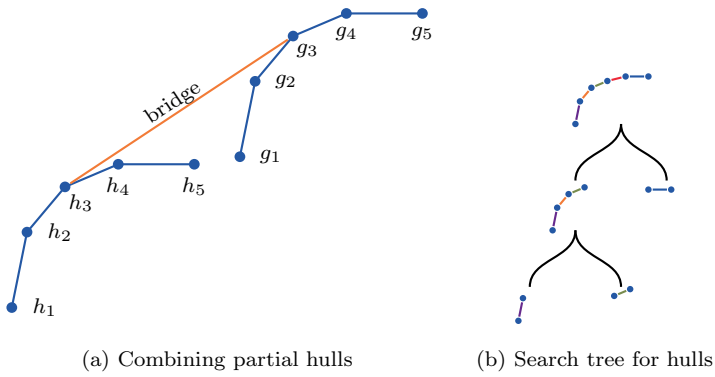

In order to compute the convex hull $C$ for a point set $P$ we can use a
conquer-and-divide technique. Assume that we have ordered the points using
the
$x$-coordinate, and split the points roughly in half, say in sets $R$ and $Q$.
Then assume we have computed convex hulls, say $H = \set{h_{i}}$ and $G =
\set{g_{i}}$, for $R$ and $Q$ independently.

A key result by~\citet{overmars1981maintenance} states that the convex hull $C$ of $P$
is equal to $\set{h_{1}, \ldots, h_{u}, g_{v}, g_{v + 1}, \ldots}$,
that is, $C$ starts with $H$ and ends with $G$.
See Figure~\ref{fig:bridge} for illustration.
The segment between $h_{u}$ and $g_{v}$ is often referred as a \emph{bridge}.

We can find the indices $u$ and $v$ in $\bigO{\log n}$ time using a binary
search over $H$ and $G$. In order to perform the binary search we will store the
hulls $H$ and $G$ in balanced search trees (red-black tree or AVL tree). Then the binary
search amounts to traversing these trees. 

Note that the concatenation and splitting of a search tree can be done in $\bigO{\log n}$ time.
In other words, we can obtain $C$ for partial convex hulls $H$ and $G$ in $\bigO{\log n}$ time.

In order to maintain the hull we will store the original points in a balanced
search tree $T$;\footnote{This is a different tree than the trees used for storing convex hulls.}
only the leaves store the actual points. Each node in $u \in T$ represents a set
of points stored in the descendant leaves of $u$.
See Figure~\ref{fig:hulltree} for illustration.

Let us write $H(u)$ to be the convex
hull of these points: we can obtain $H(u)$ from $H(\lc{u})$ and $H(\rc{u})$
in $\bigO{\log n}$ time. So whenever we modify $T$ by adding or removing a leaf $v$,
we only need to update the ancestors of $v$, and possibly some additional nodes
due to the rebalancing. All in all, we only need to update $\bigO{\log n}$ nodes, which
brings the running time to $\bigO{\log^2 n}$.

An additional complication is that whenever we compute $H(u)$ we also destroy
$H(\lc{u})$ and $H(\rc{u})$ in the process, trees that we may need in the future. However, we can rectify this by
storing the remains of the partial hulls, and then reversing the join if we
were to modify a leaf of $u$. This reversal can be done in $\bigO{\log n}$
time.

\subsection{Maintaining the convex hull of a ROC curve}

Our next step is to adapt the existing algorithm to our setting so that
we can maintain the hull of an ROC curve $X$.

First of all, adding or removing data points shifts the remaining points.
To partially rectify this issue, we will use non-normalized coordinates $R = \enpr{r_0}{r_m}$
given in Eq.~\ref{eq:rocraw}. We can do this because scaling does not
change the convex hull.

Consider adding or removing a data point $z$ which is represented by a leaf $u \in T$.
The points in $R$ associated with smaller scores than $\score{z}$ will not shift,
and the points in $R$ associated with larger scores than $\score{z}$ will shift by the same amount.
Consequently, the only partial hulls that are affected are the ancestors of $u$.
This allows us to use the update algorithm of \citet{overmars1981maintenance} for our setting as long as we can obtain the
coordinates of the points.

Our second issue is that we do not have access to the coordinates $r_i$.
We approach the problem with the same strategy as when we were computing AUC.

Let $U$ be the search tree of a convex hull $H$. Let $u \in U$ be a node
with coordinates $r_{i}$. We will define and store $\weight{u}$ as the coordinate difference $r_i - r_{i - 1}$.
Let $s_i$ be the score corresponding to $r_i$. Then Eq.~\ref{eq:rocraw} implies that $\weight{u} = \sum_{s_{i - 1} < \score{z} \leq s_i} \weight{z}$.

In addition, we will store $\cweight{u}$, the total sum of the coordinate differences
of descendants of $u$, including $u$ itself.

Let $u$ be the root of $U$.
The coordinates, say $p$, of $u$ in $U$ are $\cweight{\lc{u}} + \weight{u}$.
Moreover, the coordinates of the left
child of $u$ are
\[
	p - \weight{u} - \cweight{\rc{\lc{u}}},
\]
and the coordinates of the right child of $u$ are
\[
	p + \weight{\rc{u}} + \cweight{\lc{\rc{u}}}\quad.
\]
In other words, we can compute the coordinates of children in $U$ in constant time if
we know the coordinates of a parent.

When combining two hulls,
the binary search needed to find the bridge is based on descending $U$ from
root to the correct node. During the binary search the algorithm needs to know
the coordinates of a node which we can now obtain from the coordinates of
the parent. In summary, we can do the binary search in $\bigO{\log n}$ time,
which allow us to maintain the hull of a ROC curve in $\bigO{\log^2 n}$ time.

For completeness we present the pseudo-code for the binary search in Appendix.

\subsection{Maintaining $H$-measure}

Now that we have means to maintain the convex hull, our next step is to maintain the $H$-measure.
Note that the only non-trivial part is $L$ given in Eq.~\ref{eq:lgeneric}.

Assume that we have $n$ data points $Z$ with $n_k$ data points having class $k$.
Let $Y = \enpr{y_0}{y_m}$ be the convex hull of the ROC curve computed from $Z$.
Let $(d_1, \ldots, d_m)$ the non-normalized differences
between the neighboring points, that is,
\[
	d_{i1} = n_1(y_{i1} - y_{(i - 1)i}) \quad\text{and}\quad
	d_{i2} = n_2(y_{i2} - y_{(i - 1)2}) \quad.
\]

We will now assume that $\pi_k$ occurring in  Eq.~\ref{eq:lgeneric} are computed from the same data as the ROC curve,
that is, $\pi_k = n_k / n$.
We can rewrite the first term in Eq.~\ref{eq:lgeneric} as 
\[
\begin{split}
	\sum_{i = 0}^m \pi_1 (1 - y_{i1}) \int_{c_i}^{c_{i + 1}} cu(c)dc
	& = \frac{1}{n} \sum_{i = 0}^m  \sum_{j = i + 1}^m d_{j1}  \int_{c_i}^{c_{i + 1}} cu(c)dc \\
	& = \frac{1}{n} \sum_{j = 1}^m d_{j1}  \sum_{i = 0}^{j - 1} \int_{c_i}^{c_{i + 1}} cu(c)dc \\
	& = \frac{1}{n} \sum_{j = 1}^m d_{j1}  \int_{0}^{c_{j}} cu(c)dc\quad. \\
\end{split}
\]
Similarly, we can express the second term of Eq.~\ref{eq:lgeneric} as
\[
\begin{split}
	\sum_{i = 0}^m \pi_2 y_{i2} \int_{c_i}^{c_{i + 1}} (1 - c)u(c)dc
	& = \frac{1}{n} \sum_{i = 0}^m  \sum_{j = 1}^i d_{j2}  \int_{c_i}^{c_{i + 1}} (1 - c)u(c)dc \\
	& = \frac{1}{n} \sum_{j = 1}^m d_{j2}  \sum_{i = j}^m \int_{c_i}^{c_{i + 1}} (1 - c)u(c)dc \\
	& = \frac{1}{n} \sum_{j = 1}^m d_{j2}  \int_{c_j}^{1} (1 - c)u(c)dc\quad.\\
\end{split}
\]

If we use the beta distribution for $u$, Eq.~\ref{eq:lbeta} reduces to
\begin{equation}
\label{eq:lnode}
	L = \frac{1}{nB(1, \alpha, \beta)} \sum_{j = 1}^m d_{j1} B(c_j, \alpha + 1, \beta) + d_{j2} (B(1, \alpha, \beta + 1)  - B(c_j, \alpha, \beta + 1))\quad.
\end{equation}

Let us now consider values $c_j$. Because we assume that $\pi_k$ are estimated from the testing data,
we have $\pi_k = n_k / n$, so the values $c_j$, given in Eq.~\ref{eq:slope}, reduce to
\[
	c_j = \frac{\pi_2(y_{j2} - y_{(j - 1)2})}{\pi_2(y_{j2} - y_{(j - 1)2)}) + \pi_1(y_{j1} - y_{(j - 1)1})}
	=  \frac{\pi_2 d_{j2} / n_2}{\pi_1 d_{j1} / n_1 + \pi_2 d_{j2} / n_2}
	= \frac{d_{j2}}{d_{j1} + d_{j2}}\quad.
\]

In summary, the terms of the sum in Eq.~\ref{eq:lnode} depend \emph{only} on the coordinate differences $d_j$.
We should stress that this is only possible if we assume that $\pi_k$ are computed from the
same data as the ROC curve. Otherwise, the terms $n_k$ will not cancel out when computing $c_j$.

Let $T$ be a binary tree representing a convex hull.
The sole dependency on $d_j$ allows us to use $T$ to maintain the $H$-measure.
In order to do that, let $v \in T$ be a node with the coordinate difference $(d_1, d_2) = \weight{v}$.
Let $c = d_2/(d_1 + d_2)$. We define
\[
	\hm{v} = d_1 B(c, \alpha + 1, \beta) + d_{2} (B(1, \alpha, \beta + 1)  - B(c, \alpha, \beta + 1))\quad.
\]

We also maintain $\chm{v}$ to be the sum of $\hm{u}$ of all descendants $u$ of $v$, including $v$.
Note that maintaining $\chm{v}$ can be done in a similar fashion as $\cweight{v}$.

Finally, Eq.~\ref{eq:lnode} implies that $L = \frac{\chm{\troot{T}} } {n B(1, \alpha, \beta)}$, allowing us to maintain the $H$-measure
in $\bigO{\log^2 n}$ time.

\section{Approximating $H$-measure}
\label{sec:happrox}

In our final contribution we consider the case where $\pi_k$ are not computed from the same dataset
as the ROC curve. The consequence is that we no longer can simplify $c_j$ so that it only depends on $d_j$,
and we cannot express $L$ as a sum over the nodes of the tree representing the convex hull. 

We will approach the task differently. We will still maintain the convex hull $H$. We then select
a \emph{subset} of points from $H$ from which we compute the $H$-measure from scratch. This subset will be selected
carefully. On one hand, the subset will yield an $\epsilon$-approximation. On the other hand, the subset
will be small enough so that we still obtain polylogarithmic running time.

We start by rewriting Eq.~\ref{eq:lgeneric}.
Given a function $\funcdef{x}{[0, 1]}{\real^+}$, let us define
\[
	L_1(x) = \int_0^1 \pi_1 x(c) c u(c) dc, \quad\text{and}\quad
	L_2(x) = \int_0^1 \pi_2 x(c) (1 - c) u(c) dc\quad.
\]

Consider the values $\set{y_i}$ and $\set{c_i}$ as used in Eq.~\ref{eq:lgeneric}.
We define two functions $\funcdef{f, g}{[0, 1]}{\real^+}$
as
\begin{equation}
\label{eq:fapprox}
\begin{split}
	g(c) & = y_{i2}, \quad\text{where}\quad c_i \leq c < c_{i + 1}, \quad\text{and}\quad g(1) = 1, \\
	f(c) & = 1 - y_{i1}, \quad\text{where}\quad c_i \leq c < c_{i + 1}, \quad\text{and}\quad f(1) = 0\quad. \\
\end{split}
\end{equation}

We can now write Eq.~\ref{eq:lgeneric} as $L = L_1(f) + L_2(g)$.

We say that a function $x'$ is an $\epsilon$-approximation of a function $x$ if
$\abs{x(c) - x'(c)} \leq \epsilon x(c)$.
The following two propositions are immediate.

\begin{proposition}
Let $x'$ be an $\epsilon$-approximation of $x$, then 
\[
	\abs{L_1(x) - L_1(x')} \leq \epsilon L_1(x) \quad\text{and}\quad \abs{L_2(x) - L_2(x')} \leq \epsilon L_2(x)\quad.
\]
\end{proposition}

\begin{proposition}
Let $f$ and $g$ be defined as in Eq.~\ref{eq:fapprox},
and let $f'$ and $g'$ be respective $\epsilon$-approximations.
Define
\[
	H = 1 - \frac{L_1(f) + L_2(g)}{L_{max}} \quad\text{and}\quad
	H' = 1 - \frac{L_1(f') + L_2(g')}{L_{max}}\quad.
\]
Then $\abs{H - H'} \leq \epsilon (1 - H)$.
\end{proposition}

In other words, if we can approximate $f$ and $g$, we can also approximate the $H$-measure.
Note that the guarantee is $\epsilon(1 - H)$, that is, the approximation is more accurate
when $H$ is closer to 1, that is, a classifier is accurate.

Next we will focus on estimating $g$.

\begin{proposition}
\label{prop:l2}
Assume $\epsilon > 0$.
Let $Y$ be the convex hull of an ROC curve. Let $Q$ be a subset of $Y$
such that for each $y_i$, there is $q_j \in Q$ such that
\begin{equation}
\label{eq:approx}
	q_j = y_i \quad\text{or}\quad q_{j2} \leq y_{i2} \leq q_{(j + 1)2} \leq (1 + \epsilon) q_{j2}\quad.
\end{equation}
Let $g$ be the function constructed from $Y$ as given by Eq.~\ref{eq:fapprox}, and let
$g'$ be a function constructed similarly from $Q$.
Then $g'$ is an $\epsilon$-approximation of $g$.
\end{proposition}

\begin{proof}

Let $(c_i)$ be the slope values computed from $Y$ using Eq.~\ref{eq:slope}, and let $(c'_i)$ be the slope
values computed from $Q$.

Due to convexity of $Y$,
the slope values have a specific property that we will use several times:
fix index $j$, and let $i$ be the index such that $y_i = q_j$. Then
\begin{equation}
\label{eq:subsetslope}
	c'_j \leq c_i \quad\text{and}\quad c'_{j + 1} \geq c_{i + 1}\quad.
\end{equation}

Assume $0 < c < 1$.
Let $i$ be an index such that $c_i \leq c < c_{i + 1}$, consequently $g(c) = y_{i2}$.
Similarly, let $j$ be an index such that $c'_j \leq c < c'_{j + 1}$, so that $g'(c) = q_{j2}$.
Let $a$ be an index such that $q_j = y_a$.

If $g(c) = g'(c)$, then we have nothing to prove.  Assume $g(c) < g'(c) = q_{j2}$.

Assume $q_{j2} > (1 + \epsilon)q_{(j - 1)2}$. Then Eq.~\ref{eq:approx} implies that $y_{a - 1} = q_{j - 1}$, and so
$c_a = c'_j \leq c < c_{i + 1}$. Thus, $i \geq a$, and $g(c) = g(c_i) \geq g(c_a) = g'(c'_j) = g'(c)$,
which is a contradiction.

Assume $q_{j2} \leq (1 + \epsilon)q_{(j - 1)2}$.
Let $b$ be an index such that $y_{b} = q_{j - 1}$.
Then Eq.~\ref{eq:subsetslope} implies
\[
    c_{b + 1} \leq c_j' \leq c < c_{i + 1}\quad.
\]
Thus, $b < i$ and so $q_{(j - 1)2} = y_{b2} = g(c_b) \leq g(c_i) = g(c)$.
This leads to
\[
	\abs{g'(c) - g(c)} = q_{j2} - g(c) \leq (1 + \epsilon)q_{(j - 1)2} - g(c) \leq (1 + \epsilon) g(c) - g(c) = \epsilon g(c),
\]
proving the proposition. 

Now, assume $g(c) > g'(c) = q_{j2}$.

Assume $q_{(j + 1)2} > (1 + \epsilon)q_{j2}$.
If $y_{a + 1} \notin Q$, then Eq.~\ref{eq:approx} leads to a contradiction.
Thus $y_{a + 1} = q_{j + 1}$ and so
$c'_j \leq c < c'_{j + 1} = c_{a + 1}$. Thus, $i \leq a$, and $g(c) = g(c_i) \leq g(c_a) = g'(c'_j) = g'(c)$,
which is a contradiction.

Assume $q_{(j + 1)2} \leq (1 + \epsilon)q_{j2}$.
Let $b$ be an index such that $y_{b} = q_{j + 1}$.
Then Eq.~\ref{eq:subsetslope} implies
\[
	 c_i \leq c < c'_{j + 1} \leq c_b\quad.
\]
Thus $i < b$ or $g(c) = y_{i2} \leq y_{b2} = q_{(j + 1)2}$.
This leads to
\[
	\abs{g(c) - g'(c)} \leq q_{(j + 1)2} - q_{j2} \leq (1 + \epsilon)q_{j2} - q_{j2} = \epsilon q_{j2} = \epsilon g'(c) < \epsilon g(c),
\]
proving the proposition.\qed
\end{proof}

A similar result also holds for $L_1(f)$. We omit the proof as it is very similar to the proof of Proposition~\ref{prop:l2}.

\begin{proposition}
Assume $\epsilon > 0$.
Let $Y$ be a convex hull of a ROC curve. Let $Q$ be a subset of $Y$
such that for each $y_i$, there is $q_j \in Q$ such that
\[
	q_j = y_i \quad\text{or}\quad 1 - q_{(j + 1)1} \leq 1 - y_{i1} \leq 1 - q_{j1} \leq (1 + \epsilon) (1 - q_{(j + 1)1})\quad.
\]
Let $f$ be the function constructed from $Y$ as given by Eq.~\ref{eq:fapprox}, and let
$f'$ be a function constructed similarly from $Q$.
Then $f'$ is an $\epsilon$-approximation of $f$.
\end{proposition}

The above propositions lead to the following strategy. Only use a subset of the
ROC curve to compute the $H$-measure; if we select the points carefully, then
the relative error will be less than $\epsilon$.

Let us now focus on estimating $L_2(g)$. Assume that we have the convex hull $Y = \enset{y_0}{y_m}$ of a
ROC curve stored in a search tree $T$. 
Consider an algorithm given in Algorithm~\ref{alg:subset} which we call \algsubset.

\begin{algorithm}
\caption{$\algsubset(u, p, q, \epsilon)$, outputs truncated part of the convex hull tree. Here, $u$ is the current node,
$p$ and $q$ are the minimum and the maximum coordinates of the subtree rooted at $u$, and $\epsilon$ is the approximation guarantee.}
\label{alg:subset}
\If {$q_2 > (1 + \epsilon) p_2$} {
	$z \define p + \weight{u} +  \cweight{\lc{u}}$\;
	Report $z$\;
	$\algsubset(\lc{u}, p, z, \epsilon)$\;
	$\algsubset(\rc{u}, z, q, \epsilon)$\;
}

\end{algorithm}

The pseudo-code traverses $T$,
and maintains two variables $p$ and $q$
that bound the points of the current subtree.
If $q_2 \leq (1 + \epsilon)p_2$,
then we can safely ignore the current subtree, otherwise we output the current root,
and recurse on both children. It is easy to see that $Q = \set{y_0, y_m} \cup \algsubset(r, 0, \cweight{r})$
satisfies the conditions of Proposition~\ref{prop:l2}.

A similar traverse can be also done in order to estimate $L_1(f)$.
However, we can estimate both values with the same subset by replacing
the if-condition with $q_2 > (1 + \epsilon) p_2 \textbf{ or } 1 - q_1 > (1 + \epsilon)(1 - p_1)$.

\begin{proposition}
\algsubset runs in $\bigO{(1 + \epsilon^{-1})\log^2 n}$ time.
\label{prop:subsettime}
\end{proposition}

\begin{proof}
Given a node $v$, let us write $T_v$ to mean the subtree rooted at $v$.
Write $p_v$ and $q_v$ to be the values of $p$ and $q$ when processing $v$.

Let $V$ be the reported nodes by \algsubset.
Let $W \subseteq V$ be a set of $m$ nodes that have two reported children.
Let $\enset{h_1}{h_m}$ be the non-normalized 2nd coordinate of nodes in $W$, ordered from smallest to largest.

Fix $i$ and
let $u$ and $v$ be the nodes corresponding to $h_i$ and $h_{i + 1}$.
Assume that $v \notin T_u$.
Let $r = \rc{u}$ be the right child of $u$.
Then $T_r \cap W = \emptyset$ as
otherwise $h_i$ and $h_{i + 1}$ would not be consecutive. We have
$h_{i + 1} \geq q_{r2} > (1 + \epsilon)p_{r2} = (1 + \epsilon) h_i$.

Assume that $v \in T_u$ which immediately implies that $u \notin T_v$.
Let $r = \lc{u}$ be the left child of $v$.
Then $T_r \cap W = \emptyset$, and we have  
$h_{i + 1} = q_{r2} > (1 + \epsilon)p_{r2} \geq (1 + \epsilon) h_i$.

In summary, $h_{i + 1} > (1 + \epsilon) h_i$.  Since $\set{h_i}$ are integers, we have $h_2 \geq 1$.
In addition, $h_m \leq n$ since the original data points (from which the ROC curve is computed) do not have
weights.

Consequently, $n \geq h_m \geq (1 + \epsilon)^{m - 2}$. Solving $m$ leads to $m \in \bigO{\log_{1 + \epsilon} n}\subseteq \bigO{(1 + \epsilon^{-1})\log n}$.

Given $v \in W$, define $k(v)$ to be the number of nodes in $V \setminus W$ that have $v$ as their youngest ancestor in $W$.
The nodes contributing to $k(v)$ form at most two paths starting from $v$. Since the height of the search tree is in $\bigO{\log n}$,
we have $k(v) \in \bigO{\log n}$.

Finally, we can bound $\abs{V}$ by
\[
	\abs{V} = \sum_{v \in W} 1 + k(v) \in \bigO{m \log n } \subseteq \bigO{(1 + \epsilon^{-1})\log^2 n},
\]
concluding the proof.\qed
\end{proof}

\subsection{Speed-up}
\label{sec:speedup}

It is possible to reduce the running time of \algsubset to $\bigO{\log^2 n +
\epsilon^{-1}\log n}$. We should point out that in practice \algsubset is probably
a faster approach as the theoretical improvement is relatively modest but at the same time
the overheads increase.

There are several ways to approach the speed-up. 
Note that the source of the additional $\log n$ term is that in
the proof of Proposition~\ref{prop:subsettime}, we have $k(v) \in \bigO{\log n}$.
The loose bound is due to the fact that we
are traversing a search tree balanced on tree height.
We will modify the search procedure, so that we can show that $k(v) \in \bigO{1}$ which will give us the desired outcome.
More specifically, we would like to traverse the hull using a search tree balanced using the 2nd coordinate.

The best candidate to replace the search tree for storing the convex hull
is a weight-balanced tree~\citep{nievergelt1973binary}.  Here, the subtrees are (roughly) balanced based on the
number of children. The problem is that this tree, despite its name, does not allow weights for
nodes. Moreover, the algorithm relies on the fact that the nodes have no
weights.

It is possible to extend the weight-balanced trees to handle the weights but such modification
is not trivial. Instead we demonstrate an alternative approach that is possible
using only stock search structures.

We will do this by modifying the search tree $T$ in which the nodes correspond to the partial hulls, see Figure~\ref{fig:hulltree}.

Let $Z$ be the current set of points and let $P = \set{(s, \ell) \in Z \mid \ell = 2}$ be the points with
label equal to 2. Set $N = Z \setminus P$. We store $P$ in a tree $T$ of bounded balance;
the points are only stored in leaves. Each leaf, say $u$, also stores all
points in $N$ that follow immediately $u$.
These points are stored in a standard search tree, say $L_u$, so
that we can join two trees or split them when needed.
Any points in $N$ that
are without a preceding point in $P$ are handled and stored separately.

Note that $L_u$ correspond to a vertical line when drawing the ROC curve.
Consequently, a point in the convex hull will always be the last point in $L_u$
for some $u$.  This allows us to define the weight $\weight{u}$ of a leaf $u$
in $T$ as $(m, 1)$, where $m$ is the number of nodes in $L_u$.  We now apply
the convex hull maintenance algorithm on $T$. As always, we maintain
the cumulative weights $\cweight{u}$ for the non-leaf nodes.

In order to approximate the $H$-measure we will use a variant of \algsubset,
except that we will traverse $T$ instead of traversing the hull.
The pseudo-code is given in Algorithm~\ref{alg:subsetalt}.
At each node we output the bridge,
if it is included in the final convex hull. The condition is easy to test, we just need to make sure
that it does not overlap with the previously reported bridges.
Since we output both points of the bridge, this may lead to duplicate points, but we can prune them as a post-processing step.
Finally, we truncate the traversal
if the subtree is sandwiched between two bridges that are close enough to each other.
It is easy to see that the output of \algsubsetalt satisfies the conditions in Proposition~\ref{prop:l2}
so we can use the output to estimate $L_2(g)$. In order to estimate $L_1(f)$ we duplicate the procedure, except we swap the labels
and negate the scores which leads to a mirrored ROC curve.

\begin{algorithm}

\caption{$\algsubsetalt(u, o, p, q, \epsilon)$, outputs truncated part of the convex hull tree. Here, $u$ is the current node,
$o$ are the minimum coordinates of the subtree rooted at $u$,
$p$ and $q$ are the coordinate bounds based on already reported bridges, and $\epsilon$ is the approximation guarantee.}
\label{alg:subsetalt}

\If {$q_2 > (1 + \epsilon) p_2$} {
	$x, y \define o + $ end points of the bridge related to $u$\;
	\uIf {$y_2 > q_2$ } {
		$\algsubsetalt(\lc{u}, o, p, q, \epsilon)$\;
	}
	\uElseIf{$p_2 > x_2$} {
		$\algsubsetalt(\rc{u}, o + \cweight{\lc{u}}, p, q, \epsilon)$\;
	}
	\Else {
		Report $x$, $y$\;
		$\algsubsetalt(\lc{u}, o, p, x, \epsilon)$\;
		$\algsubsetalt(\rc{u}, o + \cweight{\lc{u}}, y, q, \epsilon)$\;
	}
}

\end{algorithm}

\begin{proposition}
\algsubsetalt runs in $\bigO{\log^2 n + \epsilon^{-1}\log n}$ time.
\end{proposition}

\begin{proof}
Let $T$ be the tree traversed by \algsubsetalt.
Let us write $T_v$ to be the subtree rooted at $v$.

Let $n(v)$ be the number of nodes in $T_v$, and let
$\ell(v)$ be the number of leaves in $T_v$. Note that $n(v) = 2\ell(v) + 1$.

Let $v$ be a child of $u$.
Since $T$ is a weight-balanced tree~\citep{nievergelt1973binary}, we have 
\begin{equation}
\label{eq:wbt}
	\alpha \leq \frac{1 + \ell(v)}{1 + \ell(u)} = \frac{1 + 1+ 2\ell(v)}{1 + 1 + 2\ell(u)} = \frac{1 + n(v)}{1 + n(u)} \leq 1 - \alpha, \quad\text{where}\quad \alpha = \frac{1 - \sqrt{2}}{2}\quad.
\end{equation}

Let us write $o(v)$ to be the 2nd origin coordinate of $T_v$.
Note that $o(v)$ corresponds to the variable $o_2$ in \algsubsetalt when $v$ is processed.

Let $V$  be the set of nodes whose bridges we output, and
let $U$ be the set of nodes in $T$ for which $\ell(u) > \epsilon o(u)$.

We will prove the claim by showing that $V \subseteq U$ and $\abs{U} \in \bigO{\log^2 n + \epsilon^{-1}\log n}$.

To prove the first claim, let $v \in V$.
Let $p$ and $q$ match the variables of \algsubsetalt when $v$ is visited.
The points $p$ and $q$ correspond to the two leaves of $T_v$.
In other words, $q_{2} - p_{2} \leq \ell(v)$, and $o(v) \leq p_{2}$.
Thus,
\[
	\ell(v) \geq q_{2} - p_{2} > \epsilon p_{2} \geq \epsilon o(v)\quad.
\]
This proves that $v \in U$.

To bound $\abs{U}$,
let $W \subseteq U$ be a set of $m$ nodes that have two children in $U$.

Define $\enpr{h_1}{h_m} = \pr{o(\rc{v}) \mid v \in W}$ to be the sequence of the (non-normalized)
2nd coordinates of the right children of nodes in $W$, ordered from the smallest to the largest. 

Fix $i$. 
Let $u \in W$ be the node for which $o(\rc{u}) = h_i$, and
let $v \in W$ be the node for which $o(\rc{v}) = h_{i + 1}$.

Assume that $h_i \leq o(v)$. Since $v \in W$, we have
\[
	h_{i + 1} = o(\rc{v}) = o(v) + \ell(\lc{v}) > o(v) + \epsilon o(\lc{v}) = (1 + \epsilon) o(v) \geq (1 + \epsilon)h_i\quad.
\]
Assume that $h_i > o(v)$. Then $u \in T_{\lc{v}}$, and consequently $v \notin T_{\rc{u}}$.
Thus, $T_{\rc{u}} \cap W = \emptyset$ as otherwise $h_i$ and $h_{i + 1}$ are not consecutive. Since $\rc{u} \in U$, we have
\[
	h_{i + 1} \geq o(\rc{u}) + \ell(\rc{u}) \geq (1 + \epsilon)o(\rc{u}) = (1 + \epsilon)h_i\quad.
\]

In summary, we have $h_{i + 1} > (1 + \epsilon) h_i$.  Note that $h_1 \geq 1$.
In addition, $h_m \leq n$ since the original data points (from which the ROC curve is computed) do not have
weights.

Consequently, $n \geq h_m \geq (1 + \epsilon)^{m - 1}$. Solving $m$ leads to 
\[
	m \in  \bigO{\log_{1 + \epsilon} n}\subseteq \bigO{(1 + \epsilon^{-1})\log n}\quad.
\]

Given $v \in W$, define $k(v)$ to be the number of nodes in $V \setminus W$ that have $v$ as their youngest ancestor in $W$.
The nodes contributing to $k(v)$ form at most two paths starting from $v$. Since the height of the search tree is in $\bigO{\log n}$,
we have $k(v) \in \bigO{\log n}$.

Assume that $\epsilon > \alpha / 2$ (recall that $\alpha = \frac{1}{2}(1 - \sqrt{2})$). Then
\[
	\abs{V} = \sum_{v \in W} 1 + k(v) \in \bigO{m \log n } \subseteq \bigO{(1 + \epsilon^{-1})\log^2 n} \subseteq \bigO{\log^2 n},
\]
proving the proposition.

Assume that $\epsilon \leq \alpha / 2$. 
Let $v \in W$ with $k(v) > 0$.
Recall that the nodes corresponding to $k(v)$ form at most two paths.
Let $u_1, \ldots, u_j$ be such a path.

Let $w$ be a child of $u_1$ for which $w \notin U$.
We have
\begin{align*}
	1 + \ell(u_1) & \leq \alpha^{-1}(1 + \ell(w))            & \text{(Eq.~\ref{eq:wbt})}\\
	& \leq \alpha^{-1}(1 + \epsilon o(w))                    & \text{($w \notin U$)}\\
	& \leq \alpha^{-1}(1 + \epsilon(o(u_1) + \ell(u_1)))     & \text{($w$ is a child of $u_1$)}\\
	& \leq \alpha^{-1}(1 + \epsilon o(u_1)) + \ell(u_1)/2,   & \text{($\epsilon \leq \alpha/2$)}\\
\end{align*}
which in turns implies $1 + \ell(u_1) \leq 2 \alpha^{-1}(1 + \epsilon o(u_1))$.

Applying Eq.~\ref{eq:wbt} iteratively and the fact that $u_j \in U$, we see that
\begin{align*}
	1 + \epsilon o(u_1) & \leq 1 + \epsilon o(u_j)                                            & \text{($u_j$ is a child of $u_1$)}\\
	                    & < 1 + \ell(u_j)                                                     & \text{($u_j \in U$)}\\
						& \leq (1 - \alpha)^{j - 1}( 1 + \ell(u_1))                           & \text{(Eq.~\ref{eq:wbt} applied $j - 1$ times)}\\
						& \leq (1 - \alpha)^{j - 1} 2 \alpha^{-1}(1 + \epsilon o(u_1))\quad.
\end{align*}
Solving for $j$ leads to
\[
	j \leq 1 + \log_{1 - \alpha} \alpha / 2 \in O(1),
\]
and consequently $k(v) \in O(1)$. We conclude that
\[
	\abs{V} = \sum_{v \in W} 1 + k(v) \in \bigO{m} \subseteq \bigO{(1 + \epsilon^{-1})\log n},
\]
proving the proposition.
\qed
\end{proof}

\section{Experimental evaluation}\label{sec:exp}

In this section we present our experimental evaluation. Our primary focus is computational time. 
We implemented our algorithm using C++.\footnote{Code is available at \url{https://version.helsinki.fi/dacs/}} 
For convenience, we refer our algorithms as \algauc, \alghexact, and \alghapprox.

We used 3 datasets obtained from UCI repository\footnote{\url{http://archive.ics.uci.edu/ml/datasets.php}}:
\dtname{APS} contains APS failure in Scania trucks,
\dtname{Diabetes} contains medical information of diabetes patients, here the label is whether the patient has been readmitted to a hospital,
\dtname{Dota2} describes the character selection and the outcome of a popular competitive online computer game.

We imputed the missing values with the corresponding means, and encoded the
categorical features as binary features. We then proceeded to train a logistic
regressor using 1/10th of the data, and used remaining data as testing data. 
When computing the $H$-measure we used beta distribution with $\alpha = \beta = 2$.

In our first experiment, we tested maintaining AUC as opposed to computing AUC
by maintaining the points sorted and computing the AUC from the sorted list~\citep{brzezinski:17:pauc}.
Given a sequence $z_1, \ldots, z_n$ of scores and labels, we compute AUC for $z_1, \ldots, z_i$
for every $i$. In the dynamic algorithm, this is done by simply adding the latest point to the existing
structure. We record the time difference after 1000 additions.

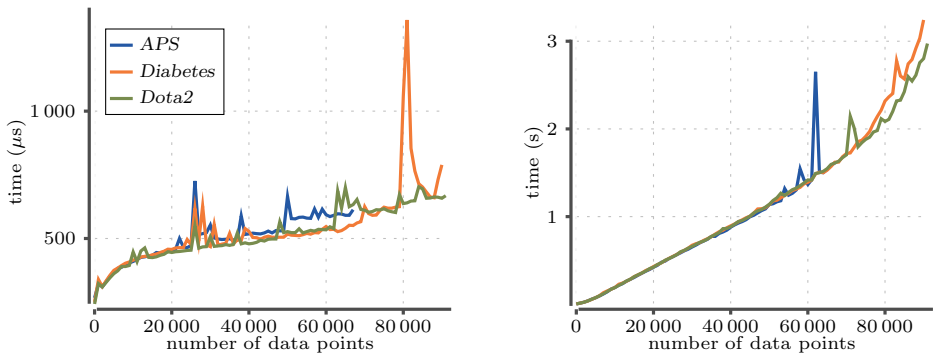
\begin{figure}
\begin{tikzpicture}
\begin{axis}[xlabel={number of data points}, ylabel= {time ($\mu$s)},
    width = 6.2cm,
    scaled x ticks = false,
    cycle list name=yaf,
    yticklabel style={/pgf/number format/fixed},
    xticklabel style={/pgf/number format/fixed},
    no markers,
	legend pos = {north west}
    ]
\addplot table[x expr = {1000*\coordindex}, y index = 0, header = false] {aps_comp_auc.txt}; 
\addplot table[x expr = {1000*\coordindex}, y index = 0, header = false] {diabetes_comp_auc.txt}; 
\addplot table[x expr = {1000*\coordindex}, y index = 0, header = false] {dota2_comp_auc.txt}; 
\legend{\dtname{APS}, \dtname{Diabetes}, \dtname{Dota2}}
\pgfplotsextra{\yafdrawaxis{1000}{92000}{250}{1400}}
\end{axis}
\end{tikzpicture}
\hfill
\begin{tikzpicture}
\begin{axis}[xlabel={number of data points}, ylabel= {time (s)},
    width = 6.2cm,
    scaled x ticks = false,
    cycle list name=yaf,
    yticklabel style={/pgf/number format/fixed},
    xticklabel style={/pgf/number format/fixed},
    no markers,
    ]
\addplot table[x expr = {1000*\coordindex}, y expr = {\thisrowno{2} / 1000000}, header = false] {aps_comp_auc.txt}; 
\addplot table[x expr = {1000*\coordindex}, y expr = {\thisrowno{2} / 1000000}, header = false] {diabetes_comp_auc.txt}; 
\addplot table[x expr = {1000*\coordindex}, y expr = {\thisrowno{2} / 1000000}, header = false] {dota2_comp_auc.txt}; 
\pgfplotsextra{\yafdrawaxis{1000}{92000}{0}{3}}
\end{axis}
\end{tikzpicture}

\caption{Running time for computing AUC 1000 times as a function of the number
of data points. Left figure: our approach. Right figure: baseline method by
computing AUC from the maintained, sorted data points.
Note that the time units are different.}
\label{fig:auc}
\end{figure}

\begin{figure}
\begin{tikzpicture}
\begin{axis}[xlabel={sliding window size}, ylabel= {time ($m$s)},
    width = 6.2cm,
    scaled x ticks = false,
    cycle list name=yaf,
    ymin = 0,
    xmin = 0,
    scaled y ticks = false,
    yticklabel style={/pgf/number format/fixed},
    xticklabel style={/pgf/number format/fixed},
    no markers,
	legend pos = {south east}
    ]
\addplot table[x expr = {1000 + 1000*\coordindex}, y expr = {\thisrowno{0} / 1000}, header = false] {aps_comp_auc_slide.txt}; 
\addplot table[x expr = {1000 + 1000*\coordindex}, y expr = {\thisrowno{0} / 1000}, header = false] {diabetes_comp_auc_slide.txt}; 
\addplot table[x expr = {1000 + 1000*\coordindex}, y expr = {\thisrowno{0} / 1000}, header = false] {dota2_comp_auc_slide.txt}; 
\legend{\dtname{APS}, \dtname{Diabetes}, \dtname{Dota2}}
\pgfplotsextra{\yafdrawaxis{0}{46000}{0}{12}}
\end{axis}
\end{tikzpicture}
\hfill
\begin{tikzpicture}
\begin{axis}[xlabel={sliding window size}, ylabel= {time (s)},
    width = 6.2cm,
    ymin = 0,
    xmin = 0,
    scaled x ticks = false,
    scaled y ticks = false,
    cycle list name=yaf,
    yticklabel style={/pgf/number format/fixed},
    xticklabel style={/pgf/number format/fixed},
    no markers,
    ]
\addplot table[x expr = {1000 + 1000*\coordindex}, y expr = {\thisrowno{1} / 1000000}, header = false] {aps_comp_auc_slide.txt}; 
\addplot table[x expr = {1000 + 1000*\coordindex}, y expr = {\thisrowno{1} / 1000000}, header = false] {diabetes_comp_auc_slide.txt}; 
\addplot table[x expr = {1000 + 1000*\coordindex}, y expr = {\thisrowno{1} / 1000000}, header = false] {dota2_comp_auc_slide.txt}; 
\pgfplotsextra{\yafdrawaxis{0}{46000}{0}{8}}
\end{axis}
\end{tikzpicture}

\caption{Running time for computing AUC $10\,000$ times in a sliding window as a function of the size
of the sliding window. Left figure: our approach. Right figure: baseline method by
computing AUC from the maintained, sorted data points.
Note that the time units are different.}
\label{fig:aucslide}
\end{figure}
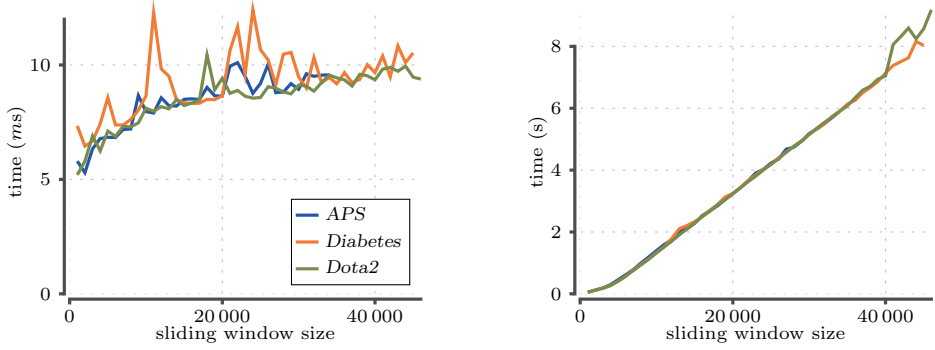

From the results shown in Figure~\ref{fig:auc} we see that \algauc is about $10^4$ times faster, though we should point out
that the exact ratio depends heavily on the implementation. More importantly, the needed time increases logarithmically for \algauc 
and linearly for the baseline. The spikes in running time of \algauc are due to self-balancing search trees.

Next, we compare the running time of computing AUC in a sliding window. We use the same baseline as in the previous
experiment, and record the running time after sliding a window for $10\,000$ steps. From the results shown in Figure~\ref{fig:aucslide} we see that \algauc
is faster than the baseline by several orders of magnitude with the needed time increasing logarithmically for \algauc
and linearly for the baseline.

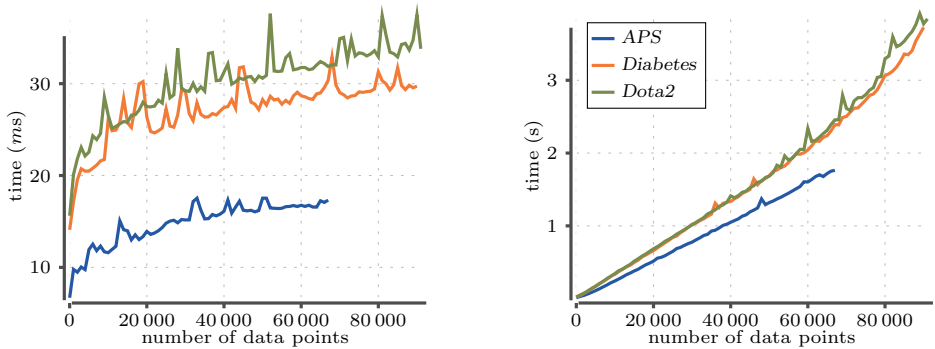
\begin{figure}
\begin{tikzpicture}
\begin{axis}[xlabel={number of data points}, ylabel= {time ($m$s)},
    width = 6.2cm,
    scaled x ticks = false,
    cycle list name=yaf,
    yticklabel style={/pgf/number format/fixed},
    xticklabel style={/pgf/number format/fixed},
    no markers,
	legend pos = {north west}
    ]
\addplot table[x expr = {1000*\coordindex}, y expr = {\thisrowno{0} / 1000}, header = false] {aps_comp_hex.txt}; 
\addplot table[x expr = {1000*\coordindex}, y expr = {\thisrowno{0} / 1000}, header = false] {diabetes_comp_hex.txt}; 
\addplot table[x expr = {1000*\coordindex}, y expr = {\thisrowno{0} / 1000}, header = false] {dota2_comp_hex.txt}; 
\pgfplotsextra{\yafdrawaxis{1000}{92000}{6.6}{35}}
\end{axis}
\end{tikzpicture}
\hfill
\begin{tikzpicture}
\begin{axis}[xlabel={number of data points}, ylabel= {time (s)},
    width = 6.2cm,
    scaled x ticks = false,
    cycle list name=yaf,
    yticklabel style={/pgf/number format/fixed},
    xticklabel style={/pgf/number format/fixed},
    no markers,
	legend pos = {north west}
    ]
\addplot table[x expr = {1000*\coordindex}, y expr = {\thisrowno{2} / 1000000}, header = false] {aps_comp_hex.txt}; 
\addplot table[x expr = {1000*\coordindex}, y expr = {\thisrowno{2} / 1000000}, header = false] {diabetes_comp_hex.txt}; 
\addplot table[x expr = {1000*\coordindex}, y expr = {\thisrowno{2} / 1000000}, header = false] {dota2_comp_hex.txt}; 
\legend{\dtname{APS}, \dtname{Diabetes}, \dtname{Dota2}}
\pgfplotsextra{\yafdrawaxis{1000}{92000}{0}{3.7}}
\end{axis}
\end{tikzpicture}
\caption{Running time for computing the $H$-measure 1000 times as a function of the number of data points. Left figure: our approach. Right figure: baseline method computing from sorted data points.
Note that the time units are different.}
\label{fig:hexact}
\end{figure}

\begin{figure}
\begin{tikzpicture}
\begin{axis}[xlabel={sliding window size}, ylabel= {time ($m$s)},
    width = 6.2cm,
    scaled x ticks = false,
    cycle list name=yaf,
    yticklabel style={/pgf/number format/fixed},
    xticklabel style={/pgf/number format/fixed},
    no markers,
	ymin = 0,
	xmin = 0,
	legend pos = {north west}
    ]
\addplot table[x expr = {1000+1000*\coordindex}, y expr = {\thisrowno{0} / 1000}, header = false] {aps_comp_hex_slide.txt}; 
\addplot table[x expr = {1000+1000*\coordindex}, y expr = {\thisrowno{0} / 1000}, header = false] {diabetes_comp_hex_slide.txt}; 
\addplot table[x expr = {1000+1000*\coordindex}, y expr = {\thisrowno{0} / 1000}, header = false] {dota2_comp_hex_slide.txt}; 
\pgfplotsextra{\yafdrawaxis{1000}{46000}{0}{650}}
\end{axis}
\end{tikzpicture}
\hfill
\begin{tikzpicture}
\begin{axis}[xlabel={sliding window size}, ylabel= {time (s)},
    width = 6.2cm,
    scaled x ticks = false,
    cycle list name=yaf,
    yticklabel style={/pgf/number format/fixed},
    xticklabel style={/pgf/number format/fixed},
    no markers,
	xmin = 0,
	legend pos = {north west}
    ]
\addplot table[x expr = {1000+1000*\coordindex}, y expr = {\thisrowno{1} / 1000000}, header = false] {aps_comp_hex_slide.txt}; 
\addplot table[x expr = {1000+1000*\coordindex}, y expr = {\thisrowno{1} / 1000000}, header = false] {diabetes_comp_hex_slide.txt}; 
\addplot table[x expr = {1000+1000*\coordindex}, y expr = {\thisrowno{1} / 1000000}, header = false] {dota2_comp_hex_slide.txt}; 
\legend{\dtname{APS}, \dtname{Diabetes}, \dtname{Dota2}}
\pgfplotsextra{\yafdrawaxis{1000}{46000}{0}{15}}
\end{axis}
\end{tikzpicture}
\caption{Running time for computing $H$-measure $10\,000$ times in a sliding window as a function of the size
of the sliding window. 
Left figure: our approach. Right figure: baseline method computing from sorted data points.
Note that the time units are different.}
\label{fig:hexactslide}
\end{figure}
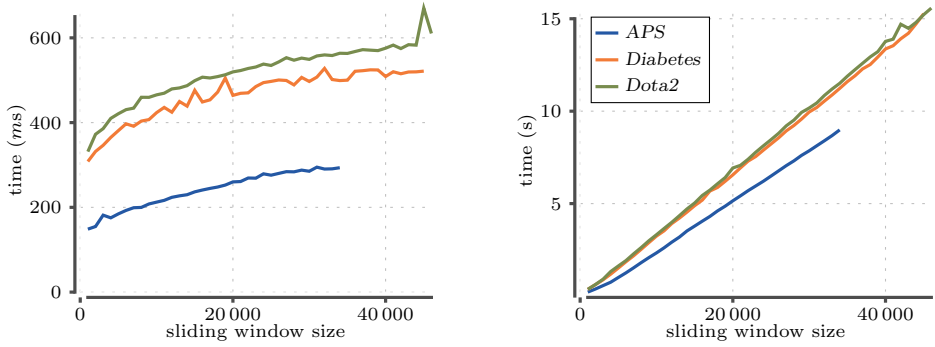

We repeat the same experiments but now we compare maintaining the $H$-measure against computing it from scratch from sorted data points.
From the results shown in Figures~\ref{fig:hexact}~and~\ref{fig:hexactslide} we see that \alghexact is
about $10$--$10^2$ times faster, and the time grows polylogarithmically for
\alghexact and linearly for the baseline. Similarly, the spikes in running
time of \alghexact are due to self-balancing search trees. Interestingly, \alghexact is faster for \dtname{APS}
than for the other datasets. This is probably due to the imbalanced labels, making the ROC curve relatively skewed, and the convex hull small.

\begin{figure}
\begin{tikzpicture}
\begin{axis}[xlabel={approximation guarantee, $\epsilon$}, ylabel= {time (s)},
    width = 6.2cm,
    scaled x ticks = false,
    cycle list name=yaf,
    yticklabel style={/pgf/number format/fixed},
    xticklabel style={/pgf/number format/fixed},
    no markers,
	ymax = 10,
	legend pos = {north west}
    ]
\addplot table[x index = 0, y expr = {\thisrowno{1} / 1000000}, header = false] {aps_happrox.txt}; 
\addplot table[x index = 0, y expr = {\thisrowno{1} / 1000000}, header = false] {diabetes_happrox.txt}; 
\addplot table[x index = 0, y expr = {\thisrowno{1} / 1000000}, header = false] {dota2_happrox.txt}; 
\pgfplotsextra{\yafdrawaxis{0}{2}{2}{10}}
\end{axis}
\end{tikzpicture}\hfill
\begin{tikzpicture}
\begin{axis}[xlabel={approximation guarantee, $\epsilon$}, ylabel={$(\text{estimate} - \text{exact}) / \text{exact}$},
    width = 6.2cm,
    scaled x ticks = false,
    scaled y ticks = false,
    cycle list name=yaf,
    yticklabel style={/pgf/number format/fixed, /pgf/number format/precision=3},
    xticklabel style={/pgf/number format/fixed},
    no markers,
	legend pos = {south west}
    ]
\addplot table[x index = 0, y expr = {(\thisrowno{2} - 0.481244) / 0.481244}, header = false] {aps_happrox.txt}; 
\addplot table[x index = 0, y expr = {(\thisrowno{2} - 0.238772) / 0.238772}, header = false] {diabetes_happrox.txt}; 
\addplot table[x index = 0, y expr = {(\thisrowno{2} - 0.057003) / 0.057003}, header = false] {dota2_happrox.txt}; 
\legend{\dtname{APS}, \dtname{Diabetes}, \dtname{Dota2}}
\pgfplotsextra{\yafdrawaxis{0}{2}{-0.03}{0}}
\end{axis}
\end{tikzpicture}

\caption{Approximative $H$-measure as a function of approximation guarantee $\epsilon$. Left figure: running time. Right figure: absolute difference to the correct value.}
\label{fig:happrox}
\end{figure}
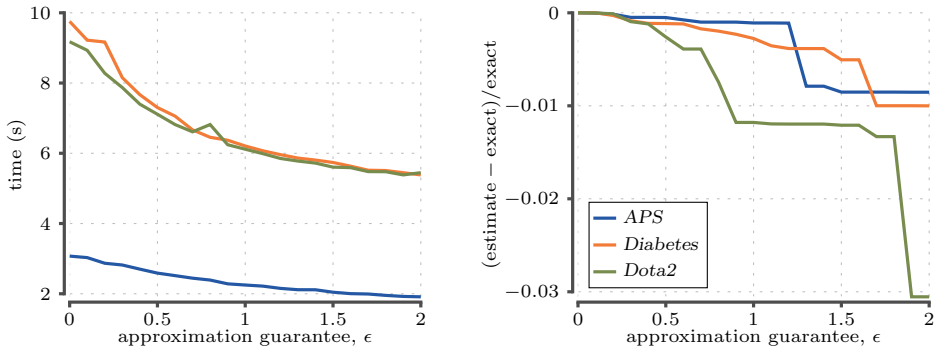

In our final experiment we use approximative $H$-measure, without the speed-up described in Section~\ref{sec:speedup}.
Here, we measure the \emph{total} time to compute the $H$-measure for $z_1, \ldots, z_i$ for every $i$ as a function of $\epsilon$.
Figure~\ref{fig:happrox} shows the running time as well as the difference to the correct score when using the whole data.

Computing the $H$-measure from scratch required roughly
1 minute for \dtname{APS}, and 2.5 minutes for \dtname{Diabetes} and \dtname{Dota2}.
On the other hand, we only need 10 seconds to obtain accurate result, and as we increase $\epsilon$,
the running time decreases. As we increase $\epsilon$, the error grows but only modestly (up to 3\%), with \alghapprox
underestimating the exact value.

\section{Conclusions}
\label{sec:conclusions}
In this paper we considered maintaining AUC and the $H$-measure under addition and deletion.
More specifically, we show that we can maintain AUC in $\bigO{\log n}$ time, and the $H$-measure
in $\bigO{\log^2 n}$ time, assuming that the class priors are obtained from the testing data.
We also considered the case, where the class priors are not obtained from the testing data.
Here, we can approximate the $H$-measure in $\bigO{(\log n + \epsilon^{-1}) \log n}$ time.

We demonstrate empirically that our algorithms, \algauc and \alghexact, provide
significant speed-up over the natural baselines where we compute the score from
the sorted, maintained data points.

When computing the $H$-measure the biggest time saving factor is maintaining
the convex hull, as the hull is typically smaller than all the data points used
for creating the ROC curve. Because of the smaller size of the hull, the tricks employed
by \alghapprox, provide less of a speed-up. Still, for larger values of $\epsilon$, the speed-up
can be almost 50\%.

\bibliographystyle{splncsnat}
\bibliography{bibliography,references}

\appendix
\section{Binary search for computing the bridge}

Algorithm~\ref{alg:bridge} contains a pseudo-code for finding the bridge of two convex hulls.
The algorithm is a variation of the search described by~\citet{overmars1981maintenance}.
The main modification here is obtaining the ROC coordinates of the points.

Due to notational convenience, we write $p \preceq q$, where $p$ and $q$ are two points in a plane,
if the slope of $p$ is smaller than or equal to the slope of $q$.

\begin{algorithm}

\caption{$\textsc{Bridge}(C_1, C_2)$, given two partial convex hulls $C_1$ and $C_2$, constructs a joint convex hull $C$ by finding
the end point of $C_1$ section and the starting point of $C_2$ section in $C$. Returns the end points and the coordinates of the bridge.}
\label{alg:bridge}

$u_1 \define $ root of $C_1$\;
$u_2 \define $ root of $C_2$\;

$s_1 \define \shift{u_1}$\;
$s_2 \define \shift{u_2} + \cdiff{u_1}$\;
$\sigma \define$ $x$-coordinate of  $\cdiff{\troot{C_1}})$\;

\While {$u_1$ or $u_2$ changes} {
	$(x, y) \define $ intersection of lines $s_1 + \diff{\nxt{u_1}} \times t$ \And $s_2 + \diff{u_2} \times t$\;
	\uIf {$\lc{u_1}$ \And $\diff{u_1} \preceq s_2 - s_1$} {
		$s_1 \define s_1 + \shift{\lc{u_1}} - \shift{u_1}$\;
		$u_1 \define \lc{u_1}$\;
	}
	\uElseIf {$\rc{u_2}$ \And $s_2 - s_1 \preceq \diff{\nxt{u_2}}$} {
		$s_2 \define s_2 + \shift{\rc{u_2}}$\;
		$u_2 \define \rc{u_2}$\;
	}
	\uElseIf {$\rc{u_1}$ \And $s_2 - s_1 \prec \diff{\nxt{u_1}}$ \And\\\quad\quad$(\Not\  \lc{u_2}$ \Or $s_2 - s_1 \preceq \diff{u_2}$ \Or $x \leq \sigma)$} {
		$s_1 \define s_1 + \shift{\rc{u_1}}$\;
		$u_1 \define \rc{u_1}$\;
	}
	\ElseIf {$\lc{u_2}$ \And $\diff{u_2} \prec s_2 - s_1$ \And\\\quad\quad$(\Not\ \rc{u_1}$ \Or $\diff{\nxt{u_1}} \preceq s_2 - s_1$ \Or $x \geq \sigma)$} {
		$s_2 \define s_2 + \shift{\lc{u_2}} - \shift{u_2}$\;
		$u_2 \define \lc{u_2}$\;
	}
}

\Return $u_1$, $u_2$, $s_1$, $s_2$\;

\end{algorithm}

\end{document}